%% file: main.tex
\documentclass[10pt,twocolumn,letterpaper]{article}

\usepackage[pagenumbers]{cvpr} 

\usepackage[accsupp]{axessibility}  

\definecolor{cvprblue}{rgb}{0.21,0.49,0.74}
\usepackage[pagebackref,breaklinks,colorlinks,allcolors=cvprblue]{hyperref}

\usepackage{amsmath}
\usepackage{amsfonts} 
\usepackage{bm}
\usepackage{amsthm}
\newtheorem{theorem}{Theorem}
\newtheorem{lemma}{Lemma}
\newtheorem{assumption}{Assumption}
\newtheorem*{theorem*}{Theorem}
\newtheorem*{lemma*}{Lemma}
\newtheorem*{assumption*}{Assumption}

\usepackage{xcolor}
\usepackage{colortbl}
\usepackage{multirow}
\usepackage{tabularx} 
\usepackage{array}
\usepackage{booktabs}
\usepackage{graphicx}
\usepackage{subcaption}   
\usepackage{algorithm}
\usepackage{algorithmic}
\usepackage{enumitem}

\def\paperID{} 
\def\confName{CVPR}
\def\confYear{2026}

\title{Curvature-Aware Zeroth-Order Optimization for\\ Memory-Efficient Test-Time Adaptation
}

\author{
Junming Zhang \quad Shuyu Yin \quad Peilin Liu \quad Rendong Ying \quad Fei Wen\thanks{Corresponding author}\\
Shanghai Jiao Tong University, China\\
{\tt\small \{hollyming,shuyu.yin,liupeilin,rdying,wenfei\}@sjtu.edu.cn}
}

\begin{document}
\maketitle
\input{sec/0_abstract}

\input{sec/1_intro}
\input{sec/2_related_work}
\input{sec/3_preliminaries}

\input{sec/4_methdology}
\input{sec/5_Theoretical_Analysis}
\input{sec/6_experiments}
\input{sec/7_conclusion}

\section*{Acknowledgements}
This work was supported by the Science and Technology Innovation (STI) 2030--Major Project under Grant 2022ZD0208700, and by the National Natural Science Foundation of China (NSFC) under Grant 62271314. 


\small
\bibliographystyle{ieeenat_fullname}
\bibliography{main}

\input{sec/X_suppl}

\end{document}

%% file: sec/0_abstract.tex
\begin{abstract}
Test-time adaptation (TTA) aims to enhance the cross-domain performance of pre-trained models by adapting to unlabeled test data.
While most existing TTA methods rely on backpropagation (BP) for finetuning, BP-free methods such as zeroth-order (ZO) methods are more desired in practical on-device scenarios. ZO methods rely only on forward computation, which can largely reduce the complexity and memory overhead of on-device deployment.
However, ZO methods suffer from much higher variance compared with first-order  methods in estimating the gradient.
To address this, we propose an improved ZO method to substantially boost the performance of ZO optimization based TTA.
First, we provide an observation to reveal the persistent low-rank Hessian structure of the loss during the adaptation process. 
Based on this insight, we then propose a loss-landscape curvature-aware zeroth-order (CAZO) method, which leverages a sliding-average estimation of the diagonal Hessian to construct a covariance matrix for anisotropic‌ perturbation sampling. 
CAZO operates by freezing pretrained weights and optimizing minimal adapter parameters via forward-only passes based gradient estimation, which can substantially reduce the memory overhead compared to BP-based methods. 
Extensive experiments demonstrate that CAZO significantly outperforms existing TTA methods, achieving state-of-the-art performance while maintaining an excellent balance between accuracy and memory efficiency. 
Code is available at \url{https://github.com/Hollyming/CAZO}.
\end{abstract}

%% file: sec/1_intro.tex
\section{Introduction}
\label{sec:intro}

Deep neural networks have demonstrated remarkable performance across various tasks, primarily due to extensive pretraining on large-scale datasets. However, when deployed in real-world applications, these models often suffer from performance degradation due to domain shifts between training and test distributions \cite{koh2021wilds, kulinski2023towards}. This issue is particularly critical in edge device applications, where data is collected in dynamic and uncontrolled environments, leading to a substantial mismatch between the pretrained model’s training domain and incoming data. Such distribution shifts can significantly degrade model performance, making it challenging to achieve reliable real-world performance.

To mitigate these issues, test-time adaptation (TTA) has emerged as a promising paradigm that allows models to adapt to out-of-distribution (OOD) test data without requiring full retraining \cite{wang2021tent, zhao2023pitfalls}. Prior TTA methods predominantly rely on BP-based optimization, using unsupervised objectives such as entropy minimization or self-supervised losses. While these methods have shown effectiveness, they incur significant computational and memory overhead, making them impractical for some resource-constrained applications, e.g., on edge devices. 
Recently, BP-free TTA methods\cite{boudiaf2022parameter, iwasawa2021test, niu2024test} have been explored to reduce memory usage. These works illustrate the potential of BP-free TTA in resource-limited scenarios. 

In this work, we explore zeroth-order (ZO) methods for TTA.
ZO methods estimate gradients through function evaluations alone and require forward-only computation, which makes them ideal for memory-constrained on-device TTA. 
However, vanilla ZO methods suffer from notorious drawbacks: their gradient estimates exhibit high variance due to random perturbations \cite{malladi2023fine} which leads to slow convergence in high-dimensional problems. It has been shown that naive ZO stochastic gradient descent may need $O(d)$ more iterations (where $d$ is the parameter dimension) to reach comparable accuracy as first-order methods \cite{nesterov2017random}. This poor efficiency makes naive ZO methods ineffective for TTA. 
Consequently, reducing the variance of ZO gradient estimators is crucial to unlock their potential for fast and effective TTA.


To reduce the variance of ZO methods for effective and efficient forward-only TTA, in this work we propose a curvature-aware zeroth-order (CAZO) optimization method for memory-efficient TTA. CAZO leverages ZO optimization to estimate gradients from forward-pass finite differences \cite{nesterov2017random, duchi2015optimal}, and cuts memory footprint by roughly 70\% while retaining competitive adaptation accuracy. 
Specifically, based on the observation that the Hessian matrix has a low-dimension structure during the adaptation process, we replace the isotropic random perturbation in the vanilla ZO with a low-cost curvature-aware update to efficiently exploit curvature information of local loss-landscape, making the adaptation more efficient and robust.

The main contributions are summarized as follows:
\begin{itemize} 
\item We analyze the curvature of the loss landscape in TTA and reveal that the Hessian admits a low-dimensional subspace that is both prominent and slowly varying throughout adaptation. This observation lays the foundation for variance reduction in ZO optimization for TTA.
\item Building on this observation, we propose a Curvature-Aware Zeroth-Order (CAZO) method, a forward-pass-only approach that utilizes a sliding-average diagonal Hessian estimation to encode the low-rank and slowly varying curvature into ZO gradient sampling. By making the ZO updates curvature-aware, CAZO substantially reduces the variance of gradient estimates and improves adaptation performance. Furthermore, we provide nonconvex convergence guarantees under standard smoothness assumptions.
\item Through extensive experiments on benchmarks like ImageNet-C and ImageNet-R/V2/Sketch, we demonstrate that CAZO achieves new state-of-the-art performance. Notably, it significantly outperforms previous methods while substantially reducing memory overhead by more than 70\% compared to BP-based methods. Our method achieves both high accuracy and high memory efficiency, which enables practical TTA on devices with limited memory resources. 
\end{itemize}

\begin{figure*}[t]
\centering
\includegraphics[width=0.8\textwidth]{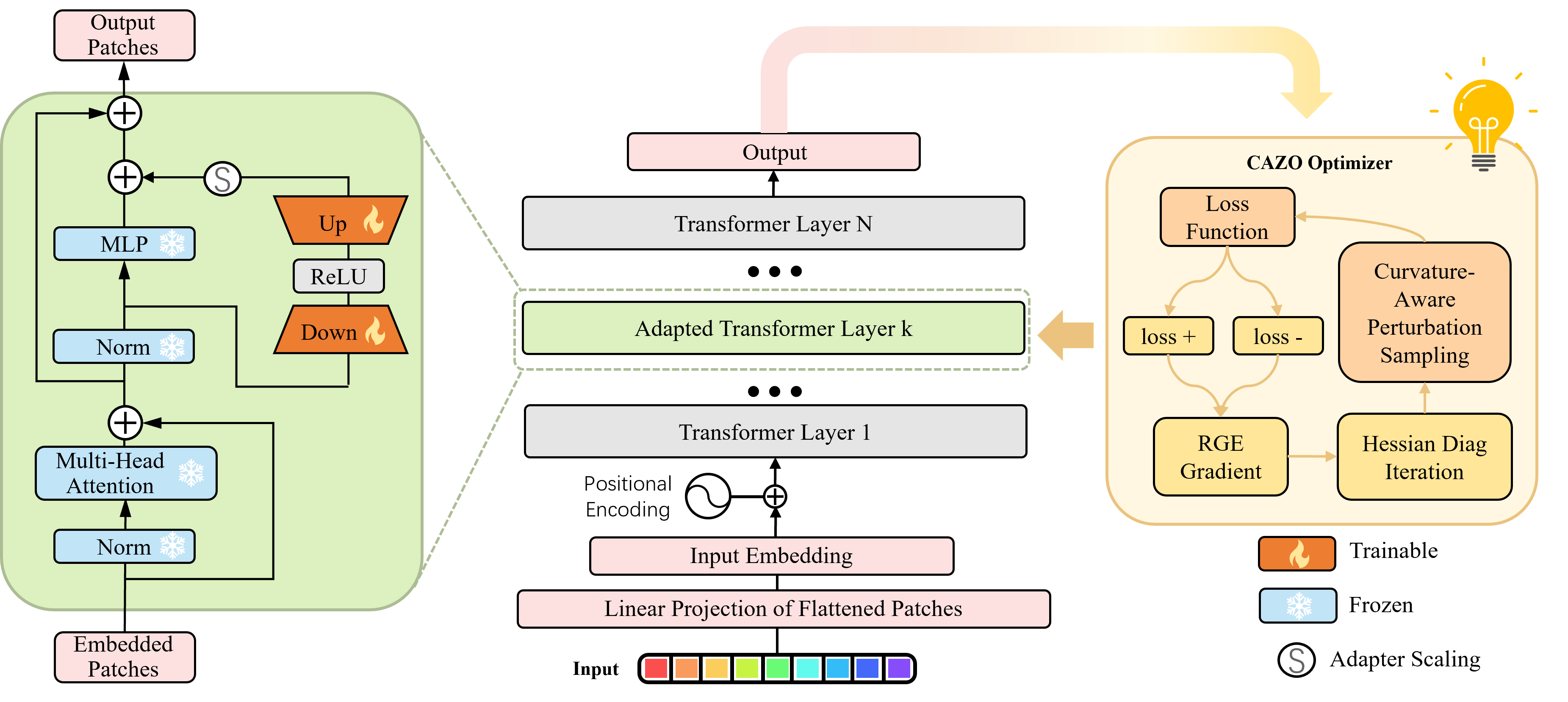}
\caption{Illustration of CAZO and the used model architecture. A lightweight adapter is updated in TTA.}
\label{fig:CAZO_framework}
\end{figure*}

%% file: sec/2_related_work.tex
\section{Related work}
\label{sec:related_work}

\subsection{Test-Time Adaptation (TTA)}  
TTA \cite{wang2021tent, sun2020test, liang2024comprehensive, wang2025distribution, improv, ada,aea,tca,duan2025brain} enables pre-trained models to adapt to new, unseen data during inference without retraining from scratch. Since incoming data is typically unlabeled, TTA is essential for applications with shifting data distributions, such as edge devices or when test data differs significantly from training data. 
Recent TTA advancements can be classified into three main categories: \textit{1) Entropy Minimization-Based Methods:} These methods minimize entropy during inference. TENT \cite{wang2021tent} updates BatchNorm (BN) statistics using entropy minimization, a concept extended by DELTA \cite{zhao2023delta}. SHOT \cite{shot} further improves adaptation by combining entropy minimization with diversity regularization.  \textit{2) Continual and Robust Test-Time Adaptation:} These approaches focus on dynamic adaptation in evolving domains. CoTTA \cite{wang2022continual} introduces a teacher-student framework with consistency loss, while EcoTTA \cite{ecotta} enhances memory efficiency through meta-networks. BECoTTA \cite{becotta} employs a mixture-of-domain low-rank experts for domain-adaptive routing, and TTACOPE \cite{ttacope} combines pretraining and TTA for object pose estimation. Methods like SOTTA \cite{sotta} and TTN \cite{ttn} target adaptation in noisy and shifting domains. \textit{3) Energy-Based and Non-Parametric Methods:} Energy-based models and non-parametric approaches are explored for TTA. AEA \cite{yuan2024tea} introduces energy-based adaptation, and AdanPC \cite{zhang2023adanpc} focuses on reliable sample selection. 

\subsection{BP-Free Methods} 
BP-free learning algorithms update model parameters without gradient backpropagation, relying instead on forward passes, sampling, or heuristics. This makes them suitable for memory and compute limited environments.
A prominent subclass is ZO optimization, which estimates gradients using only function evaluations. ZO methods have been widely applied in domains where gradients are inaccessible, such as black-box adversarial attacks \cite{ilyas2018black, tu2019autozoom, shu2023zeroth, zhang2022robustify, verma2023certified, zhao2019design}, model explanation \cite{dhurandhar2019model}, and automated ML pipelines \cite{gu2021optimizing, wang2022zarts}. However, their gradient estimation often suffers from high variance in high-dimension problems, with slow convergence. MeZO \cite{malladi2023fine} shows that the convergence depends more on effective intrinsic dimension rather than raw dimensionality. Further, the work \cite{guo2024zeroth,zhao2024second} exploits the sparsity in parameter updates to enhance the effectiveness of ZO methods in finetuning LLMs.



During the final preparation of this work, we became aware of a concurrent ZO TTA method, ZOA \cite{deng2025test}. While both ZOA and CAZO employ ZO updates, ZOA targets quantized models and emphasizes domain-knowledge management for long-term adaptation, whereas CAZO focuses on curvature-aware sampling for improving ZO methods in TTA. Under similar memory budget, experiments show that CAZO achieves stronger performance.



Another class is based on evolutionary strategies (ES) \cite{hansen2001completely}, which optimize parameters via population-based search and stochastic sampling, without relying on gradient information. For TTA, FOA \cite{niu2024test} employs covariance matrix adaptation evolution strategy\cite{hansen2016cma} to optimize prompt vectors using forward passes only. ES methods have been widely used in reinforcement learning and black-box optimization to demonstrate strong performance under strict memory or non-differentiable constraints.
In addition, surrogate-based optimization approaches such as Bayesian Optimization \cite{snoek2012practical, shahriari2015taking} have emerged as another category of BP-free algorithms. These methods iteratively build probabilistic models to guide exploration of  parameter space, and have been used in various model and hyperparameter selection tasks where gradients are unavailable or expensive to compute.

In addition, surrogate-based optimization approaches such as Bayesian Optimization \cite{snoek2012practical, shahriari2015taking} have emerged as another category of BP-free algorithms. These methods iteratively build probabilistic models to guide exploration of parameter space, and have been used in various model and hyperparameter selection tasks where gradients are unavailable or expensive to compute.


%% file: sec/3_preliminaries.tex
\section{Preliminaries on Zeroth-Order Optimization}
\label{sec:preliminaries}


ZO methods approximate gradients by evaluating model performance at perturbed parameter points. In the commonly used ZO approach known as random gradient estimation (RGE) \cite{nesterov2017random,duchi2015optimal}, model parameters $\theta$ are iteratively updated by perturbing them in random directions and using the resulting differences in loss to approximate the gradient. The RGE gradient estimation is
\begin{equation}
\hat{\nabla} {\cal L}(\theta) = \frac{1}{k} \sum_{i=1}^{k}  \frac{{\cal L}(\theta + \epsilon {u}_i) - {\cal L}(\theta - \epsilon {u}_i)}{2\epsilon} {u}_i,
\label{RGE}
\end{equation}
where ${\theta}$ denotes the model parameters, ${\cal L}({\theta})$ represents the loss function, \( \epsilon \) controls the perturbation magnitude, and \( \{u_i\}_{i=1,\cdots,k} \) are independent random vectors drawn from a standard gaussian \( \mathcal{N}(0, I) \) or uniform distribution. $k$ is the number of perturbations for each gradient estimation. Averaging over $k$ perturbations refines the gradient estimate, thereby reducing estimate variance and improving approximation accuracy.
Using the ZO-SGD method \cite{ghadimi2013stochastic}, the model parameters are iteratively updated as
\begin{equation}
    {\theta}_{t+1}={\theta}_t-\alpha\hat{\nabla}{\cal L}({\theta}_t),
    \label{ZO-SGD}
\end{equation}
where $\alpha$ denotes the learning rate, and $\hat{\nabla}{\cal L}$ represents the ZO gradient estimation at time $t$ via \eqref{RGE}.

Although ZO methods are BP-free, they often suffer from slow convergence due to their high variance in gradient estimates. Specifically, the variance of the RGE estimator in \eqref{RGE} scales linearly with the parameter dimension, i.e., of order $O(d/k)$ \cite{duchi2015optimal}. Therefore, it is crucial to reduce the variance of ZO estimation to make it effective for adapting neural networks with high-dimensional parameters.



%% file: sec/4_methdology.tex
\section{Methodology}
\label{sec:methodology}

\begin{figure}[t]
\centering
    \begin{minipage}{0.4\textwidth}  
        \centering
        \includegraphics[width=\linewidth]{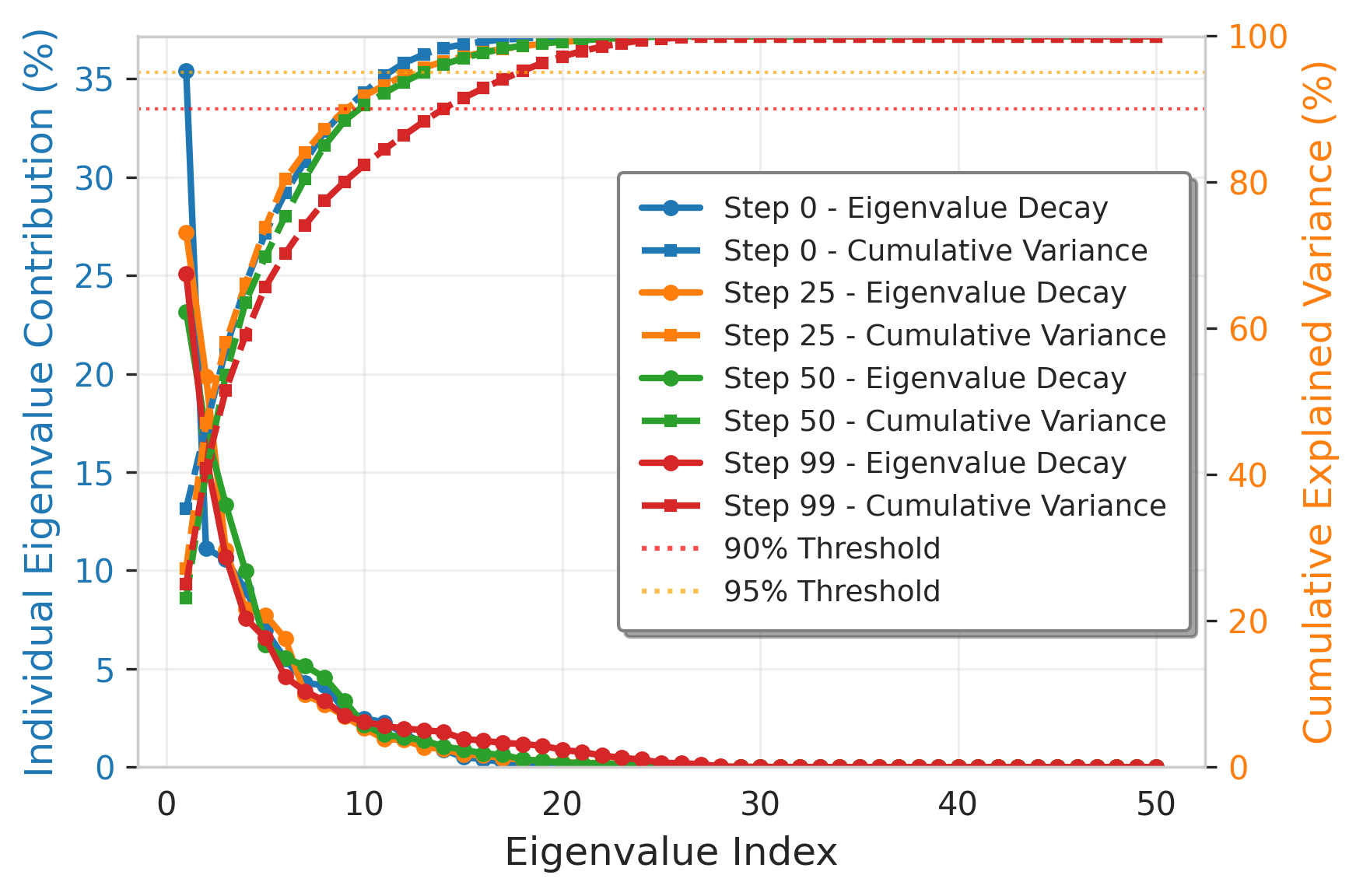}
        \caption{Low-rank structure of Hessian during the adaptation process.}
        \label{fig:hessian_low_rank}
    \end{minipage}
\end{figure}

We first reveal that the Hessian of loss function exhibits a prominent low-rank structure during adaptation. Then, we propose an enhanced ZO method for TTA that exploits the loss curvature information to construct anisotropic sampling for achieving better performance.

\subsection{Properties of Hessian in TTA}
\paragraph{The persistent low-rank property of Hessian in TTA.}
Consider a standard TTA setting where a pretrained model $f(\theta)$ is adapted to unlabeled test data $\{x_t\}_{t=1}^T$ from a shifted distribution $\mathcal{D}_{\mathrm{test}}$. In this work, we focus on updating a small subset of parameters in adapter $\theta_{\mathrm{adapt}}\subset\theta$. 



We begin by empirically analyzing the Hessian of the loss function during the adaptation process. Specifically, we investigate the TTA of a pretrained ViT-B/16 model on the ImageNet-C dataset under Gaussian corruption with severity level 5. At different adaptation steps $\{0, 25, 50, 99\}$ out of 100 total iterations, we compute the empirical Hessian of the loss function with respect to the adapter parameters and examine its eigenvalue spectrum.

As shown in Fig.~\ref{fig:hessian_low_rank}, the eigenvalue spectrum exhibits consistent and prominent low-rank structure across adaptation steps. For example, the top 20 eigenvalues account for more than 96\% of the total variance, and the effective rank occupies only 0.22\% of the total parameter dimensionality. Moreover, both the eigenvalue decay and cumulative explained variance remain highly stable over time. 
These observations collectively indicate that the strong low-rank structure of the Hessian persists throughout the adaptation process.
\begin{figure}[t]
    \centering
    \includegraphics[width=0.8\linewidth]{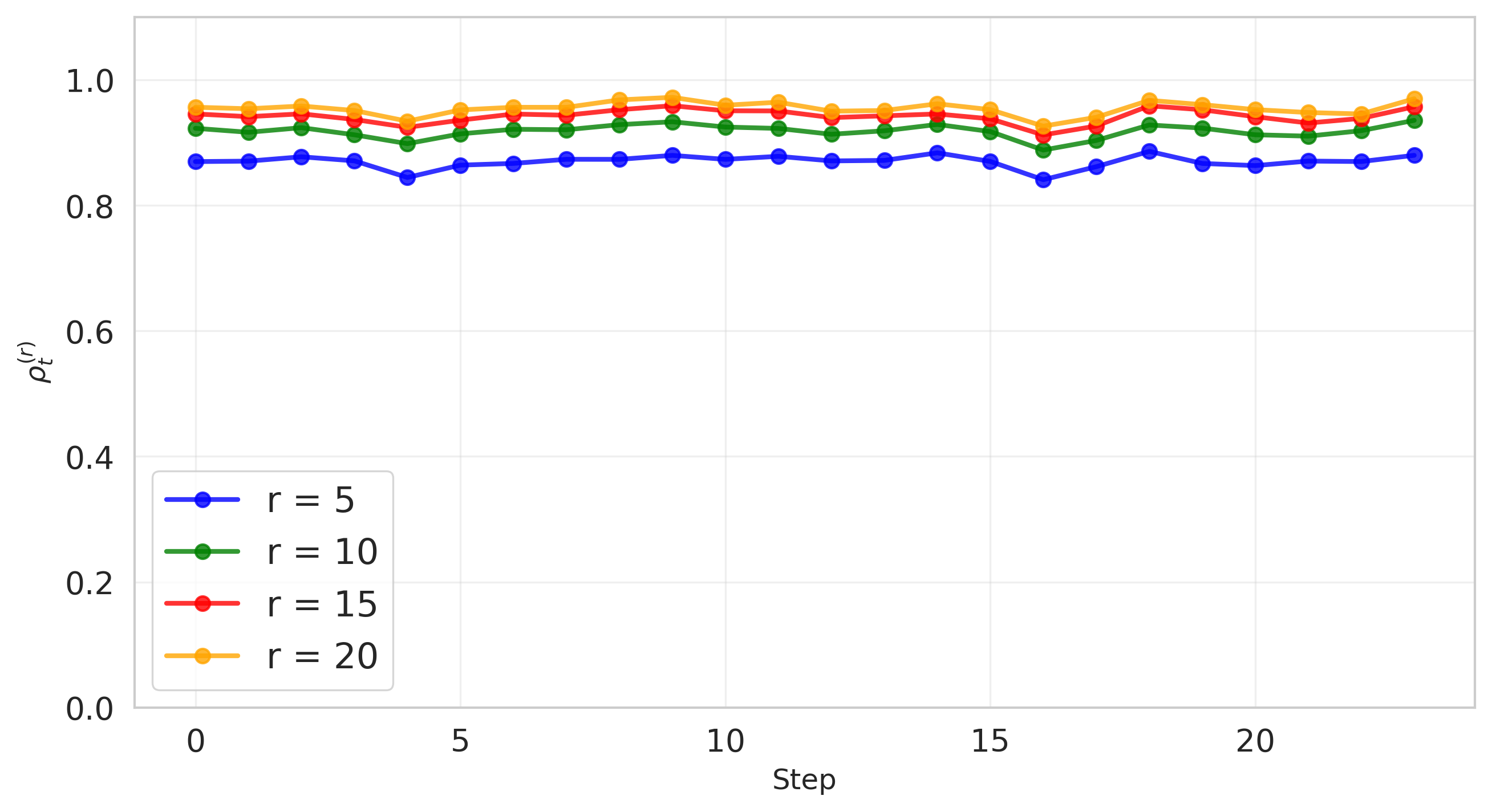}
    \caption{
    Projection ratio $\rho_t^{(r)}$ between adjacent Hessians during TTA for different dimensions of the principal curvature subspace, e.g., $r\in\{5,10,15,20\}$. The principal curvature subspace remains highly consistent (slow-varying) across steps.
    }
    \label{fig:projection_ratio_evolution}
\end{figure}

\paragraph{The slow-varying property of the principal components of Hessian in TTA.}
While the above analysis reveals a persistent low-rank structure of the Hessian during the adaptation process, next we examine how the dominant directions of the loss curvature evolve during adaptation.
Specifically, for each adaptation step $t$, we compute the empirical Hessian 
${H}_t$ with respect to the adapter parameters and extract its top-$r$ eigenvectors 
${U}_t^{(r)}$. 
The corresponding projection matrix is
\begin{equation}
{P}_t = {U}_t^{(r)} ({U}_t^{(r)})^\top.
\end{equation}
We then measure how much the next-step Hessian ${H}_{t+1}$ is preserved in the principal subspace of the previous-step Hessian ${H}_{t}$ by computing the \emph{projection ratio}:
\begin{equation}
\rho_t^{(r)} = 
\frac{\|{P}_t {H}_{t+1}\|_F}{\|{H}_{t+1}\|_F}.
\end{equation}
A high value of $\rho_t^{(r)}$ indicates that the dominant curvature directions 
remain consistent across adjacent steps. 

As shown in Figure~\ref{fig:projection_ratio_evolution}, for a batch size of 512 and 25 adaptation steps, the projection ratio stays around $0.9$ throughout adaptation for different ranks $r\in\{5,10,15,20\}$, which demonstrates that the principal curvature subspace varies slowly over time during adaptation.
We further observe that larger batch sizes lead to a smoother trajectory of $\rho_t^{(r)}$, 
which suggests that the expected loss landscape also possesses low-rank and slow-varying properties. 
This slow-varying property supports our design choice of using an EMA-based update rule in estimating the Hessian in the proposed method: the exponential averaging effectively tracks the smooth curvature drift in a temporally stable subspace, yielding more robust adaptation dynamics.

\subsection{Curvature-Aware Zeroth-Order Optimization for TTA}

\begin{figure}[t]
    \centering
    \includegraphics[width=\linewidth]{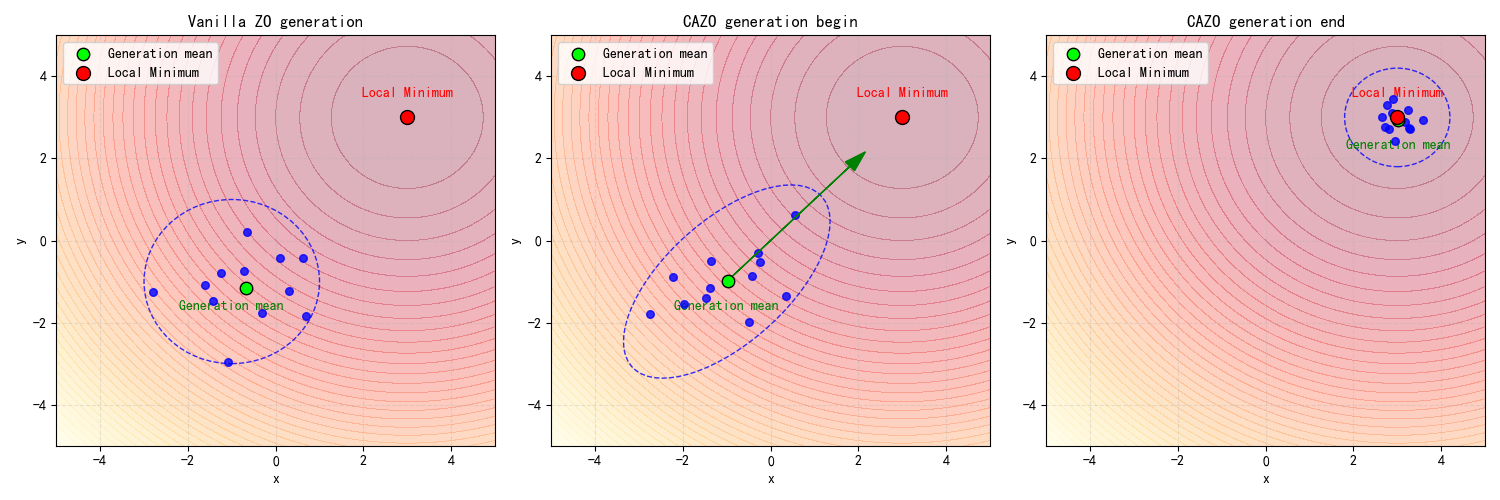}
    \caption{Illustration of curvature-aware perturbation generation for more efficient ZO optimization.}
    \label{fig:perturbation_generate}
\end{figure}

The Hessian properties observed above naturally motivate a curvature-aware design of ZO sampling. The persistent low-rank structure indicates that only a small number of directions dominate the loss curvature during TTA, while most dimensions reside in relatively flat regions. Moreover, the slow evolution of the Hessian's principal subspace shows that these curvature-dominant directions remain stable across adaptation steps and can thus be reliably estimated online. Such geometric regularities suggest that uniform isotropic perturbations used in standard ZO methods inefficiently allocate sampling effort, which wastes evaluations in uninformative directions while insufficiently probing curvature-sensitive ones.
These insights indicate that incorporating curvature information into the perturbation distribution can improve the efficiency of ZO optimization for TTA. Ideally, the sampling covariance should reduce variance along high-curvature directions and enlarge it along low-curvature ones, which produces more informative function evaluations and lower-variance gradient estimation.

Guided by these observations, we introduce CAZO, which adaptively shapes perturbation directions according to the local loss curvature. Specifically, CAZO samples perturbations from a preconditioned Gaussian distribution that attenuates perturbations in sharp directions and amplifies them in flat ones, as illustrated in Figure \ref{fig:perturbation_generate}. 
Specifically, the gradient estimator of CAZO at iteration $t$ is given by
\begin{equation}
\label{eq:hessian_grad_estimator}
\begin{aligned}
    \hat{g}({\theta}_{t}) = \frac{1}{k}&\sum\limits_{i = 1}^k \frac{{\cal L}\left( {\theta}_t + \epsilon {u}_i \right) - {\cal L}\left( {\theta}_t - \epsilon {u}_i \right)}
    {2\epsilon} {u}_i, \\
    &\mathrm{~with~} u_i \sim \mathcal{N}(0, \tilde{H}^{-1}_t).
\end{aligned}
\end{equation}
where $\tilde{H}_t$ denotes an estimated curvature matrix at step $t$.
Direct computation of $\tilde{H}^{-1}$ is computationally intractable for high-dimensional neural networks. To address this, we adopt a diagonal approximation the covariance matrix, i.e., $\Sigma =\tilde{H}^{-1} = \operatorname{diag}(\sigma^2_{1},\cdots,\sigma^2_{d}) \succ 0$, which significantly improves computational and memory efficiency in practice.

To estimate the diagonal curvature $\tilde{H}_t$, we employ a sliding exponential moving average (EMA) of the element-wise squared ZO gradients

\begin{equation}
\label{eq:EMA_hessian}
\begin{aligned}
D_t &=(1-\nu)D_{t-1}+\nu \hat{g}^2(\theta_{t-1}),\quad \\
\tilde{H}_t&=\mathrm{diag}\left(\frac{D_t}{1-(1-\nu)^t}\right),
\end{aligned}
\end{equation}
where $\hat{g}^2$ denotes the element-wise square of the gradient estimate in Eq.~\eqref{eq:hessian_grad_estimator} and $0 \leq \nu \leq 1$ is the EMA coefficient.

In implementation, CAZO uses a composite loss function consisting of an unsupervised entropy loss on the test data and an MSE loss for feature alignment using feature statistics from clean data \cite{niu2024test} as in Fig.~\ref{fig:CAZO_framework}. 
Finally, the adapter parameters are updated using the ZO estimated gradients in a standard SGD. A sketch of the procedure of CAZO is summarized in Algorithm \ref{alg:cazo}.

\begin{algorithm}[t]
\caption{CAZO Algorithm}
\label{alg:cazo}
\footnotesize
\begin{algorithmic}[1]
\renewcommand{\algorithmicrequire}{\textbf{Input:}}
\renewcommand{\algorithmicensure}{\textbf{Output:}}
\REQUIRE Model parameters $\theta$, number of iterations $T$, learning rate $\alpha$
\STATE Initialize diagonal Hessian approximation ${D_0} \leftarrow I$ \COMMENT{Identity matrix}
\FOR {$t = 1 \to T$}
    \STATE Sample $K$ perturbations: $\{u_{i}^{t}\}_{i=1}^{K} \sim \mathcal{N}(0, \tilde{H}^{-1}_t)$
    \FOR {$i = 1 \to K$}
        \STATE Compute positive loss: $\ell_{i}^{+} \leftarrow \mathcal{L}(\theta_t + \epsilon u_{i}^{t})$
        \STATE Compute negative loss: $\ell_{i}^{-} \leftarrow \mathcal{L}(\theta_t - \epsilon u_{i}^{t})$
    \ENDFOR
    \STATE Estimate gradient: $\tilde{g}({\theta}_{t}) \leftarrow \frac{1}{K} \sum_{i=1}^{K} \left( \ell_{i}^{+} - \ell_{i}^{-} \right) \frac{u_{i}^{t}}{2\epsilon}$
    \STATE Update Hessian diagonal: $\tilde{H}^{-1}_t \leftarrow (D_{t-1}, \hat{g}(\theta_{t}))$ by Eqn.\eqref{eq:EMA_hessian}
    \STATE Update parameters: $\theta_{t+1} \leftarrow \theta_t - \alpha \hat{g}(\theta_{t})$
\ENDFOR
\ENSURE Updated parameters $\theta$
\end{algorithmic}
\end{algorithm}

%% file: sec/5_Theoretical_Analysis.tex
\section{Convergence analysis}
\label{sec:convergence}

We provide a convergence guarantee for CAZO under standard smoothness and variance assumptions. Let $\mathcal{L}(x;\theta)$ be the loss on data $x \in \mathbb{R}^n$ with model parameters $\theta \in \mathbb{R}^d$. The ZO gradient estimator of CAZO is defined in Eq.~\eqref{eq:hessian_grad_estimator}, where perturbations are drawn from $\mathcal{N}(0, \widetilde{H}_t^{-1})$, and $\widetilde{H}_t$ is updated via EMA as in Eq.~\eqref{eq:EMA_hessian}.
We further denote $\nabla f(\theta) = \mathbb{E}_{x}[\nabla f(x;\theta)].$
We make the following standard assumptions:
\begin{assumption}[L-smoothness]
    \label{assum:L_smooth}
    Assume the loss function $\mathcal{L}(x;\theta)$ is $L$-smooth respect to parameter $\theta$.
\end{assumption}

\begin{assumption}[Data variance]
    \label{assum:data_var}
    Assume the data variance satisfies $\mathbb{E}_{x}[\| \nabla f(x;\theta) - \nabla f(\theta) \|] \leq \sigma^2$.
\end{assumption}

\begin{assumption}[Value range of $\widetilde{H}^{-1}_t$]
    \label{assum:hessian_value_range}
    Assume for $t \geq 0$, the element in $\tilde{H}^{-1}_t$ are in range $[\beta_l, \beta_u]$ where $0 < \beta_l \leq \beta_u < \infty$.
\end{assumption}

We first bound the bias and variance of the CAZO gradient estimator:
\begin{lemma}[Estimation error of the zeroth order gradient estimator]
    \label{lemma:mean_and_variance}
    Given a loss function $\mathcal{L}(\theta_t)$ with parameter $\theta_t$ satisfies Assumption \ref{assum:L_smooth}, and a data $x_t$ sampled from $\mathcal{D}$, a gradient estimator and variance matrix estimator in Eq.~\eqref{eq:hessian_grad_estimator} and Eq.~\eqref{eq:EMA_hessian} satisfies Assumption \ref{assum:data_var} and \ref{assum:hessian_value_range}, the following relation holds
    \begin{equation*}
        \begin{aligned}
        \mathbb{E}_{x_t,u_i}\left[ \hat{g}(x_t;\theta_t) \right] & = \tilde{H}^{-1}_t \nabla \mathcal{L}(\theta_t) + \mathcal{O}(\epsilon)\\
        \mathbb{E}_{x_t,u_i}\left[ \| \hat{g}(x_t;\theta_t) \|^2 \right] & \leq 2d(d+2)\beta_u \left(\| \nabla \mathcal{L}(\theta_t) \|^2 + \sigma^2 \right)\\
        &+ \mathcal{O}(\epsilon^2).
        \end{aligned}
    \end{equation*}
\end{lemma}
Based on Lemma~\ref{lemma:mean_and_variance}, we derive the convergence rate below:
\begin{theorem}[Convergence rate of CAZO]
    Given a loss function $\mathcal{L}$ satisfies Assumption \ref{assum:L_smooth}, a gradient estimator and variance matrix estimator in Eq.~\eqref{eq:hessian_grad_estimator} and Eq.~\eqref{eq:EMA_hessian} satisfies Assumption \ref{assum:data_var} and \ref{assum:hessian_value_range}, with training time $T$, learning rate $\eta = \frac{\beta_l}{2Ld(d+2) \beta_u \sqrt{T}}$, the convergence rate of the CAZO is 
    \begin{equation*}
        \begin{split}
            \frac{1}{T} \sum_{t=1}^T \mathbb{E}\left[ \| \nabla \mathcal{L}(\theta_t) \|^2\right] & \leq \frac{4Ld(d+2)\beta_u\left(\mathcal{L}(\theta_{t}) - \mathcal{L}(\theta^*) \right)}{\beta_l^2 \left( \sqrt{T} - 1\right)} \\
            &~~+ \frac{\sigma^2}{\sqrt{T} - 1} + \mathcal{O}(\epsilon^2).
        \end{split}
    \end{equation*}
\end{theorem}


This result establishes an $\mathcal{O}(1/\sqrt{T})$ convergence rate, in line with classical ZO methods, while making the convergence constants explicitly depend on the curvature bounds $\beta_l$ and $\beta_u$. In particular, when the curvature proxy is well conditioned (e.g., $\beta_l$ is not too small and $\beta_l^2 > \beta_u$), the bound suggests a smaller constant factor than isotropic ZO sampling, hinting at improved sample efficiency in practice. Detailed proofs can be found in the supplementary material (Sec.~\ref{sec:convergence}).


%% file: sec/6_experiments.tex
\section{Experiments}
\label{sec:experiments}

We evaluate our method on ImageNet-C \cite{hendrycks2018benchmarking}, a comprehensive benchmark for TTA that contains 15 corruption types with five severity levels, as well as three domain-shifted datasets: ImageNet-R \cite{hendrycks2020augmix}, ImageNet-V2 \cite{recht2019imagenet}, and ImageNet-Sketch \cite{wang2019learning}. We use ViT-B/16 \cite{Dosovitskiy2021vit} as the source model. 
For a fair comparison, we include both non-BP methods (LAME \cite{boudiaf2022parameter}, T3A \cite{iwasawa2021test}, FOA \cite{niu2024test}, ZOA\cite{deng2025test}), the zeroth-order baseline ZO (RGE), and BP-based methods (TENT \cite{wang2021tent}, CoTTA \cite{wang2022continual}, SAR \cite{niu2023towards}, DeYO\cite{lee2024entropy}, EATA\cite{niu2022efficient}, RoTTA\cite{yuan2023robust}). 
We evaluate the performance in terms of classification accuracy on the most severe corruption level (severity-5). More detailed experimental settings are provided in supplemental material (Sec.~\ref{SM:experiment_settings}). 

\begin{table*}[t]
\caption{Comparison of accuracy (\%, $\uparrow$) over 15 corruptions on the ImageNet-C dataset (severity level 5) with ViT-B/16 (with model reset for each corruption). The results are reported as mean $\pm$ std.}
\label{tab:results}
\centering
\resizebox{\textwidth}{!}{
\begin{tabular}{l|c|ccc|cccc|cccc|cccc|c}
\hline
&& \multicolumn{3}{c|}{Noise} & \multicolumn{4}{c|}{Blur Defoc.} & \multicolumn{4}{c|}{Weather} & \multicolumn{4}{c|}{Digital} & {Average} \\
Method & BP & Gauss. & Shot & Impul. & Defoc. & Glass & Motion & Zoom & Snow & Frost & Fog & Brit. & Contr. & Elas. & Pix. & JPEG & Acc.($\%$)\\
\hline
NoAdapt & $\times$ & 56.8 & 56.8 & 57.5 & 46.9 & 35.6 & 53.1 & 44.8 & 62.2 & 62.5 & 65.7& 77.7 & 32.6 & 46.0 & 67.0 & 67.6 & 55.5 \\
LAME & $\times$ & 56.5 & 56.5 & 57.2 & 46.4 & 34.8 & 52.7 & 44.2 & 58.6 & 61.5 & 63.1& 77.4 & 24.7 & 44.6 & 66.6 & 67.2 & 54.1\\
T3A & $\times$ & 56.4& 56.9 & 57.3& 47.9 & 37.7 & 54.2 & 46.9 & 63.5& 60.8& 68.4 & 78.1 & 38.3 & 50.0 & 67.6 & 68.9 & 56.9 \\
FOA & $\times$ & 61.5 & 62.8 & 63.3 & 58.5 & 53.9 & 61.0 & 56.7 & 69.4 & 68.9 & 73.8 & 80.9 & 67.2 & 62.7 & 73.6 & 72.8 & 65.8$\pm$0.1 \\
ZOA & $\times$ & 61.6 & 63.1 & 63.5 & 59.7 & 59.0 & 64.9 & 62.6 & 70.4 & 68.4 & 74.0 & 80.6 & 67.1 & 69.0 & 74.7 & 73.2 & 67.5$\pm$0.2\\
TENT & $\checkmark$ & 60.3 & 61.5 & 61.8 & 59.1 & 56.6 & 63.5 & 59.1 & 56.7 & 64.3 & 2.7 & 79.2 & 67.4 & 61.3 & 72.7 & 70.6 & 59.8$\pm$0.2 \\
CoTTA & $\checkmark$ & 62.4 & 63.5 & 63.8 & 54.7 & 51.3 & 64.1 & 56.4 & 69.0 & 68.3 & 72.4 & 78.5 & 15.6 & 63.4 & 73.6 & 71.3 & 61.9$\pm$1.2\\
SAR & $\checkmark$ & 59.2 & 60.5 & 60.7 & 57.5 & 55.6 & 61.8 & 57.6 & 65.9 & 63.5 & 69.1 & 78.7 & 45.7 & 62.4 & 71.9 & 70.3 & 62.7$\pm$0.1 \\
DeYO & $\checkmark$ & 59.6 & 60.7 & 60.3 & 57.7 & 58.1 & 63.9 & 61.7 & 68.3 & 65.8 & 70.7 & 78.6 & 51.6 & 68.7 & 73.8 & 71.5 & 64.7$\pm$2.4\\
EATA & $\checkmark$ & \textbf{62.8} & 64.3 & 64.2 & \textbf{61.1} & \textbf{61.5} & \textbf{66.8} & 64.7 & 71.3 & 69.3 & 63.3 & 80.4 & 55.4 & 71.0 & \textbf{75.9} & 73.6 & 66.8$\pm$2.3\\
RoTTA & $\checkmark$ & 57.6 & 58.1 & 58.6 & 47.5 & 37.8 & 54.5 & 46.1 & 63.1 & 63.2 & 66.7 & 77.8 & 30.1 & 48.4 & 67.6 & 67.7 & 56.3$\pm$0.1\\
\rowcolor{gray!20}
CAZO & $\times$ & 62.7 & \textbf{64.2} & \textbf{64.3} & 60.1 & 61.3 & 66.2 & \textbf{65.2} & \textbf{72.9} & \textbf{70.8} & \textbf{74.2} & \textbf{81.3} & \textbf{69.3} & \textbf{72.5} & 75.8 & \textbf{74.4} & \textbf{69.0} $\pm$ 0.1 \\
\hline
\end{tabular}
}
\end{table*}

\begin{table*}[t]
  \centering
  \caption{Comparison of accuracy (\%, $\uparrow$) over 15 corruptions on the ImageNet-C dataset (severity level 5) with ViT-B/16 in a continual adaptation setting without model reset.}
  \label{tab:imagenetc_vitb16}
  \resizebox{\textwidth}{!}{%
  \begin{tabular}{l|ccc|cccc|cccc|cccc|c}
    \toprule
    \multicolumn{17}{c}{$t \rightarrow$} \\
    \midrule
    & \multicolumn{3}{c|}{Noise} & \multicolumn{4}{c|}{Blur Defoc.} & \multicolumn{4}{c|}{Weather} & \multicolumn{4}{c|}{Digital} & {Average} \\
    Method & Gauss. & Shot & Impul. & Defoc. & Glass & Motion & Zoom & Snow & Frost & Fog & Brit. & Contr. & Elas. & Pix. & JPEG & Acc.($\%$)\\
    \midrule
    Tent  & 59.8 & 63.4 & 62.9 & 45.5 & 50.1 & 58.6 & 50.9 & 63.3 & 60.5 & 66.5 & 78.6 & 55.4 & 54.6 & 69.7 & 70.2 & 60.7 \\
    CoTTA & 59.9 & 62.9 & 62.7 & 44.4 & 48.3 & 55.6 & 47.4 & 61.2 & 65.5 & 52.2 & 74.3 & 23.5 & 67.5 & 72.5 & 71.5 & 58.0 \\
    ETA   & 59.5 & 63.7 & 63.2 & 52.3 & 52.4 & 58.6 & 55.6 & 64.6 & 62.5 & 63.7 & 77.9 & 50.3 & 59.4 & 70.1 & 71.4 & 61.7 \\
    RoTTA & 57.7 & 60.0 & 60.4 & 41.9 & 34.3 & 51.5 & 43.5 & 64.7 & 63.4 & 35.8 & 78.2 & 21.7 & 47.7 & 67.0 & 68.1 & 53.1 \\
    SAR   & 59.1 & 61.2 & 61.5 & 54.2 & 55.3 & 58.5 & 55.8 & 60.9 & 61.9 & 64.6 & 76.8 & 58.3 & 58.2 & 68.4 & 68.9 & 61.6 \\
    DeYO  & 59.2 & 61.6 & 60.8 & 44.6 & 47.9 & 56.8 & 49.3 & 61.8 & 61.6 & 62.8 & 77.0 & 56.3 & 54.6 & 66.2 & 69.7 & 59.4 \\
    LCoTTA  & 60.6 & 62.8 & 62.5 & 52.2 & 55.4 & 58.9 & 54.5 & 65.3 & 63.4 & 66.2 & 78.6 & 54.5 & \textbf{60.7} & 69.2 & 69.2 & 62.3 \\
    \rowcolor{gray!20}
    CAZO & \textbf{62.6} & \textbf{64.8} & \textbf{64.5} & \textbf{55.9} & \textbf{59.0} & \textbf{63.2} & \textbf{60.3} & \textbf{66.9} & \textbf{68.4} & \textbf{70.3} & \textbf{79.2} & \textbf{64.7} & 57.5 & \textbf{70.2} & \textbf{71.5} & \textbf{65.3}$\pm$1.0\\
    \bottomrule
  \end{tabular}}
  \vspace{-10pt}
\end{table*}

\subsection{Main Results}
\begin{table}[htbp]
\caption{Performance comparison on ImageNet-R/V2/Sketch with ViT-B/16. }
\label{tab:other_datasets}
\centering
\resizebox{0.7\columnwidth}{!}{%
\begin{tabular}{@{}l|c|ccc>{\columncolor{gray!10}}c}
\hline
&& \multicolumn{4}{c}{\textbf{Accuracy (\%), $\uparrow$}}\\
Method &BP & R & V2 & Sketch & Avg.\\ 
\hline
NoAdapt & $\times$ & 59.5 & 75.4 & 44.9 & 59.9 \\
LAME & $\times$ & 59.0 & 75.3 & 44.4 & 59.6 \\
T3A & $\times$ & 55.4 & 75.6 & 48.3 & 59.8 \\
FOA & $\times$ & 63.2 & 75.3 & 49.6 & 62.7 \\
TENT & $\checkmark$ & 63.9 & 75.1 & 49.0 & 62.6 \\
CoTTA & $\checkmark$ & 63.4 & 75.5 & 50.0 & 63.0 \\
SAR & $\checkmark$ & 63.8 & 75.1 & 48.7 & 62.5 \\
DeYO & $\checkmark$ & 68.0 & 73.5 & 49.8 & 63.8\\
RoTTA & $\checkmark$ & 59.7 & 75.4 & 45.2 & 60.1\\
\hline
CAZO & $\times$ & 64.4 & 75.3 & 50.9 & 63.5 \\
\hline
\end{tabular}
}
\vspace{-10pt}
\end{table}

Table~\ref{tab:results} reports the results on ImageNet-C (severity level 5) with full-precision ViT-B/16 under the standard protocol where the model is reset for each corruption. CAZO achieves the best performance with an average accuracy of 69.0\%, clearly outperforming both BP-free and BP-based competitors. 
Among BP-free baselines, methods such as LAME and T3A bring only marginal gains over NoAdapt (55.5\%), while the more recent FOA and ZOA are considerably stronger: FOA leverages a CMA-ES strategy to evolve prompt vectors and reaches 65.8\%, and ZOA employs a zeroth-order update with a domain-knowledge offset bank to obtain 67.5\%. Nevertheless, CAZO further improves the average accuracy, yielding gains of +3.2\% and +1.5\% over FOA and ZOA, respectively, under a comparable memory budget. On the BP-based side, TENT, SAR, CoTTA and DeYO all benefit from gradient-based updates (e.g., SAR and CoTTA achieve 62.7\% and 61.9\%), yet CAZO still surpasses them by +6.3\% and +7.1\%, while avoiding any backward passes. Taken together, these results underscore the superior performance of CAZO, highlighting the ability to achieve high accuracy with minimal memory overhead. 

Table~\ref{tab:imagenetc_vitb16} presents results in the more challenging continual TTA (CTTA) scenario, where the model traverses all corruptions once without reset. In this setting we compare against strong CTTA methods such as TENT, SAR, RoTTA, DeYO and LCoTTA\cite{duan2025lifelong}, as well as ETA, the variant of EATA that does not require source-domain Fisher information and is therefore compatible with the standard CTTA protocol. CAZO again attains the highest robustness with an average accuracy of 65.3\%, outperforming LCoTTA (62.3\%), ETA (61.7\%) and SAR (61.6\%) by +3.0, +3.6 and +3.7 points, respectively. The consistent improvements in both the model-reset and continual settings indicate that curvature-aware sampling not only stabilizes zeroth-order updates but also remains effective when distribution shifts accumulate over time.

Moreover, from the results on ImageNet-R/V2/Sketch in Table~\ref{tab:other_datasets},  CAZO achieves competetive performance, which further demonstrates its effectiveness.
Fig.~\ref{fig:acc_loss} further illustrates the effectiveness of CAZO, showing that it not only achieves the fastest loss descent but also maintains consistently high accuracy throughout the adaptation process. These trends provide strong empirical support for our theoretical variance reduction analysis.

\subsection{Memory and Runtime Efficiency}
Moreover, our approach achieves a 4-10$\times$ reduction in memory usage compared to BP-based methods. As reported in Table~\ref{tab:comparison}, under full-precision settings, TENT requires 6,404 MB of runtime CUDA memory, whereas CAZO requires only 1,695 MB. This substantial memory saving is especially pronounced when compared to methods such as CoTTA, which incur additional overhead due to data augmentation techniques.
As further shown in Table~\ref{tab:comparison}, CAZO demonstrates excellent scalability with respect to the number of perturbations \(k\). Notably, the memory usage remains nearly constant across all \(k \in \{2, 4, 8, 20\}\), owing to our use of a diagonal curvature proxy, lightweight adapter parameterization, and efficient memory management.

In terms of runtime, CAZO exhibits a favorable trade-off between speed and accuracy: for example, with \(k=20\), CAZO achieves 69.0\% accuracy in 3,127 seconds, which is comparable to FOA’s 2,885 seconds runtime but significantly outperforms it in accuracy (65.8\% for FOA). These results further highlight CAZO's practicality for scalable, memory-efficient test-time adaptation.

\begin{figure}[t]
    \centering
    \begin{subfigure}[b]{0.48\columnwidth}
        \centering
        \includegraphics[width=\linewidth]{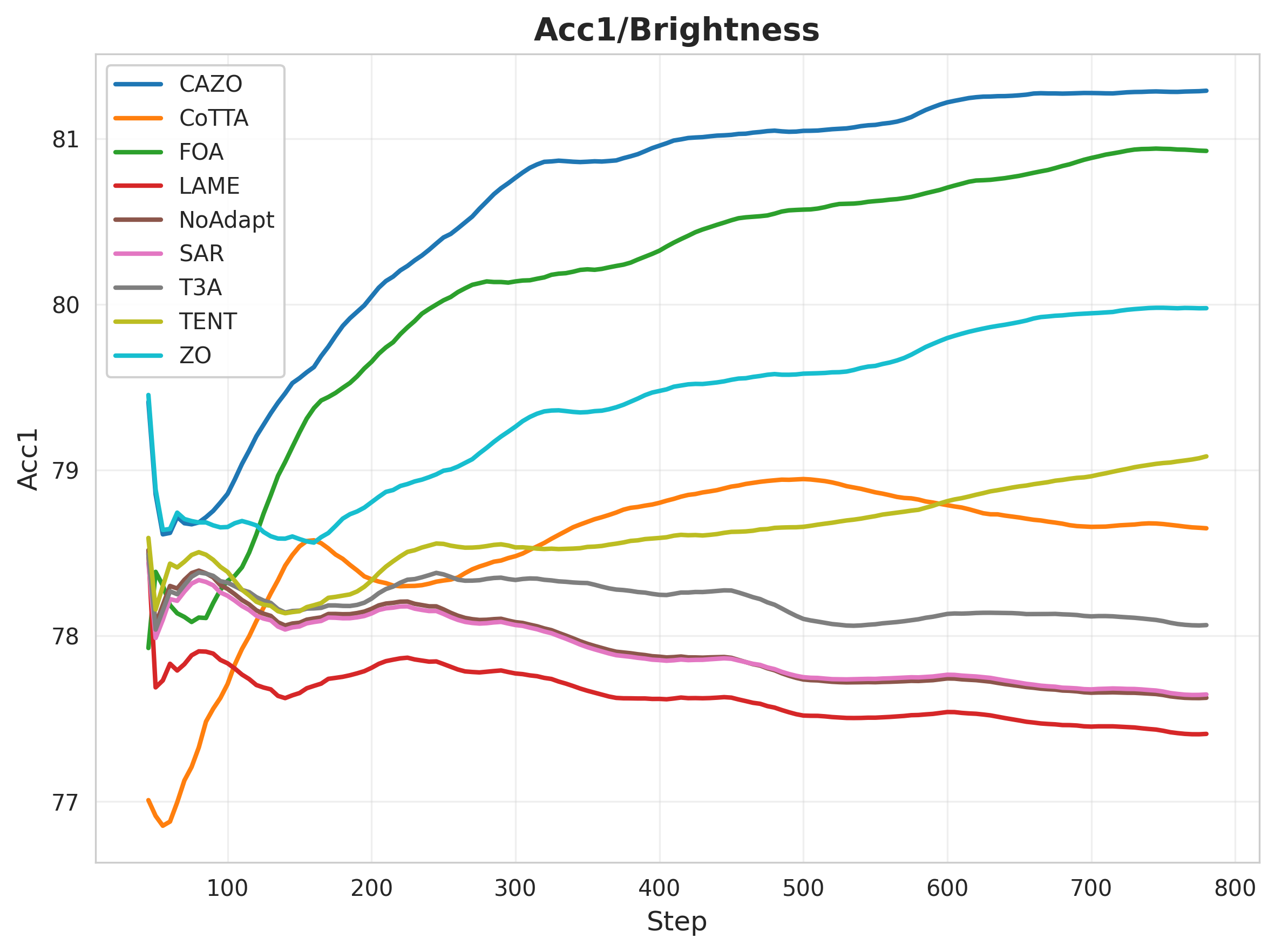}
        \caption{Accuracy on Brightness}
        \label{fig:acc_brightness}
    \end{subfigure}
    \begin{subfigure}[b]{0.48\columnwidth}
        \centering
        \includegraphics[width=\linewidth]{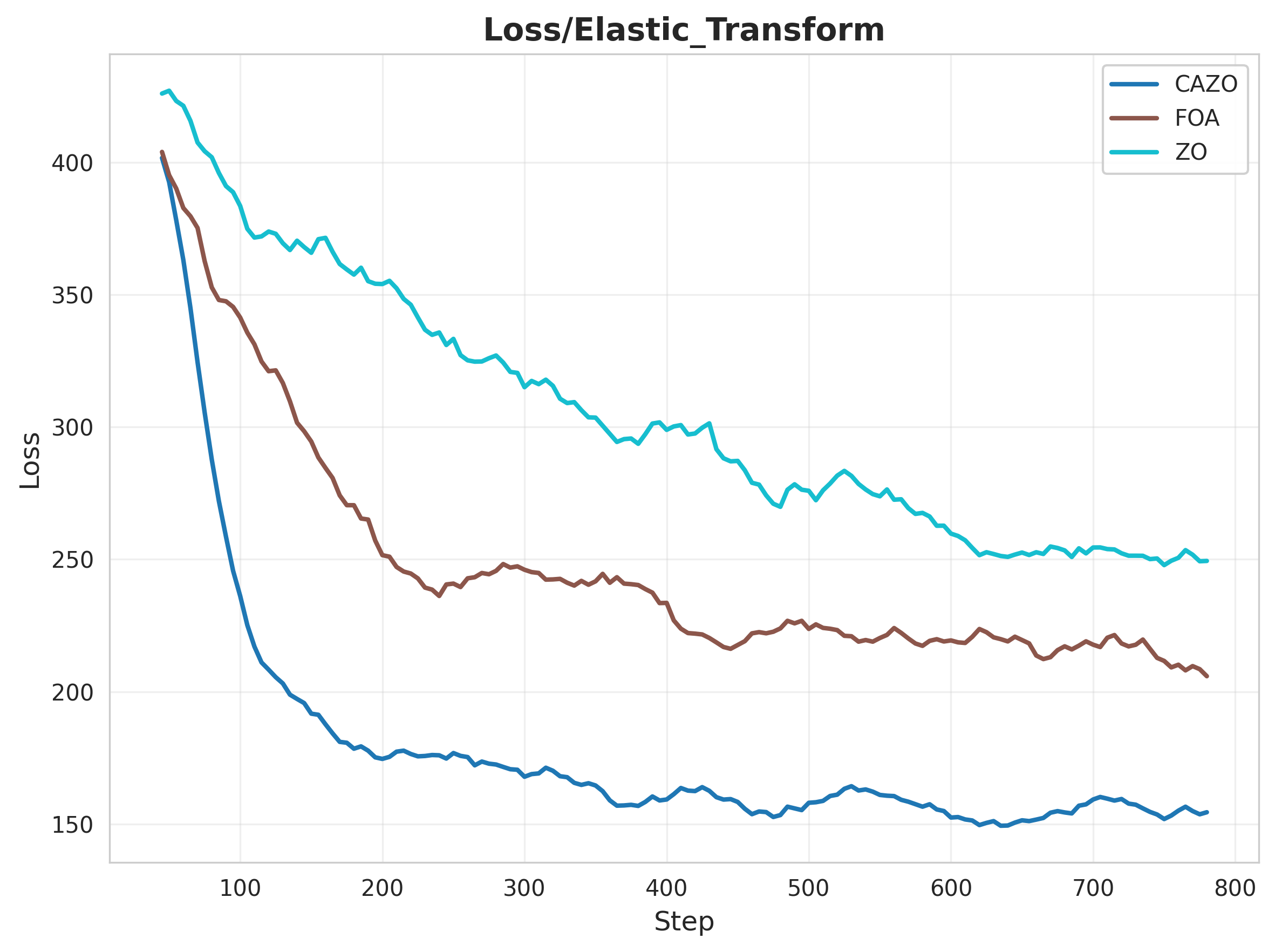}
        \caption{Loss on Elas}
        \label{fig:loss_elastic_transform}
    \end{subfigure}
    \caption{Learning curves of ViT-B/16 on ImageNet-C (severity level 5). (a) Accuracy on brightness corruption; (b) Loss on elastic transform corruption.}
    \label{fig:acc_loss}
\end{figure}


\begin{table}[t]
\caption{Computation complexity comparison. FP/BP denote forward/backward propagation, with \#FP and \#BP indicating the number of forward/backward passes per sample. Accuracy (\%) and ECE (\%) are averaged over ImageNet-C (level 5) using ViT-B/16. Wall-clock time (seconds) and memory usage (MB) are measured on 50,000 ImageNet-C samples using a single NVIDIA H20 GPU. The perturbation number $k$ for ZO-based methods is set within $[2, 20]$, as analyzed in Fig.~\ref{fig:perturbations} and $p$ denotes the population size in FOA. }
\label{tab:comparison}
\centering
\resizebox{\columnwidth}{!}{%
\begin{tabular}{l|c|cc|cc|cc}
\hline
\rowcolor{gray!10}
&&&&\multicolumn{2}{c|}{average}&Time&Memory\\
\rowcolor{gray!10}
Method & BP & \#FP & \#BP & Acc. & ECE &(seconds) &(MB) \\
\hline
NoAdapt & $\times$ & 1 & 0 & 55.5 & 10.5 & 93 & 1,543 \\
T3A & $\times$ & 1 & 0 & 56.9 & 26.9 & 176 & 1,853 \\
FOA ($p=28$) & $\times$ & 28 & 0 & 65.8 & 3.1 & 2,885 & 1,553 \\
ZOA & $\times$ & 2 & 0 & 67.5 & 4.8 & 398 & 1,660\\
TENT & $\checkmark$ & 1 & 1 & 59.8 & 18.3 & 210 & 6,404 \\
SAR & $\checkmark$ & [1,2] & [0,2] & 62.7 & 10.4 & 246 & 6,405 \\
CoTTA & $\checkmark$ & 3 or 35 & 1 & 61.9 & 6.8 & 961 & 17,773 \\
DeYO & $\checkmark$ & [1,2] & [0,1] & 64.7 & 12.8 & 511 & 7,124\\
EATA & $\checkmark$ & 1 & [0,1] & 66.8 & 14.9 & 346 & 6,696\\
RoTTA & $\checkmark$ & 2 & 1 & 56.3 & 10.3 & 804 & 9,126\\
ZO  ($k=20$)& $\times$ & 40 & 0 & 62.9 & 5.6 & 3,166 & 1,695  \\
\hline
CAZO ($k=2$)& $\times$ & 4 & 0 & 65.2 & 6.1 & 417 & 1,693 \\
CAZO ($k=4$)& $\times$ & 8 & 0 & 67.2 & 5.4 & 663 & 1,693 \\
CAZO ($k=8$)& $\times$ & 16 & 0 & 67.9 & 4.7 & 1,260 & 1,695  \\
CAZO ($k=20$)& $\times$ & 40 & 0 & 69.0 & 4.3 & 3,127 &  1,695 \\
\hline
\end{tabular}}
\vspace{-10pt}
\end{table}

\subsection{Evaluation on Quantized Model}
We further evaluate CAZO and other methods under 8-bit and 6-bit quantization settings to test its compatibility with edge-device deployment. As shown in Table~\ref{tab:quant_imagenet_c_results}, CAZO maintains its performance advantage even under aggressive quantization. Notably, in the 8-bit setting, CAZO achieves 67.8\% accuracy, outperforming all other methods and preserving most of its full-precision performance. Even in the 6-bit case, CAZO retains 61.2\% accuracy which is significantly higher than other methods, demonstrating its robustness and efficiency in adapting low-bit quantized models. These results suggest that CAZO is well suited for real-world deployment on resource-limited hardware.

\begin{table*}[htbp]
\caption{Results on adapting quantized ViT-B/16 (with 8-bit and 6-bit) on ImageNet-C.}
\label{tab:quant_imagenet_c_results}
\centering
\resizebox{\textwidth}{!}{
\begin{tabular}{l|l|ccc|cccc|cccc|cccc|>{\columncolor{gray!10}}c}
\hline
&& \multicolumn{3}{c|}{Noise} & \multicolumn{4}{c|}{Blur Defoc.} & \multicolumn{4}{c|}{Weather} & \multicolumn{4}{c|}{Digital} &{Average}\\
quant &Method & Gauss. & Shot & Impul. & Defoc. & Glass & Motion & Zoom & Snow & Frost & Fog & Brit. & Contr. & Elas. & Pix. & JPEG & Acc.(\%) \\
\hline
\multirow{6}{*}{8-bit}& NoAdapt & 55.9 & 55.8 & 56.8 & 46.1 & 35.0 & 52.8 & 43.6 & 61.1 & 61.0 & 66.4 & 77.1 & 22.6 & 44.9 & 66.4 & 66.8 & 54.1\\
& T3A & 55.7 & 56.0 & 56.5 & 47.1 & 37.0 & 53.8 & 45.8 & 62.6 & 59.3 & 68.4 & 77.4 & 32.5 & 48.4 & 67.0 & 68.2 & 55.7\\
& FOA & 60.6 & 61.5 & 61.5 & 57.0 & 47.3 & 58.6 & 52.4 & 67.2 & 67.7 & 72.8 & 80.0 & 62.3 & 58.3 & 71.6 & 70.1 & 63.3 \\
& LAME & 55.7 & 55.6 & 56.5 & 45.7 & 34.2 & 52.4 & 42.9 & 57.1 & 59.7 & 65.3 & 76.8 & 8.0 & 43.5 & 66.0 & 66.4 & 52.4\\
& ZO & 61.0 & 61.7 & 61.1 & 54.1 & 53.0 & 59.6 & 51.3 & 68.3 & 65.9 & 69.0 & 79.9 & 46.3 & 55.3 & 70.1 & 70.6 & 61.8 \\
& CAZO & 61.8 & 63.3 & 63.3 & 58.7 & 60.1 & 64.9 & 63.8 & 72.1 & 69.9 & 73.2 & 80.8 & 64.9 & 71.8 & 74.9 & 73.9 & \textbf{67.8} \\
\hline
\multirow{6}{*}{6-bit} & NoAdapt & 44.0 & 42.6 & 44.7 & 38.2 & 28.7 & 43.5 & 35.7 & 52.6 & 60.5 & 59.0 & 75.0 & 25.4 & 38.9 & 59.7 & 65.2 & 47.6\\
& T3A & 43.1 & 42.0 & 44.0 & 39.6 & 31.6 & 44.7 & 38.4 & 56.0 & 60.5 & 63.1 & 75.2 & 24.3 & 43.2 & 60.5 & 66.2 & 48.8\\
& FOA & 49.5 & 50.1 & 53.3 & 47.9 & 35.9 & 48.9 & 44.9 & 59.0 & 64.5 & 68.2 & 76.0 & 42.6 & 48.3 & 65.3 & 67.0 & 54.8\\
& LAME & 43.8 & 42.3 & 44.5 & 37.7 & 27.9 & 42.9 & 34.9 & 43.9 & 59.2 & 54.1 & 74.7 & 24.6 & 37.2 & 59.3 & 64.8 & 46.1\\
& ZO & 48.5 & 48.2 & 51.5 & 46.3 & 43.9 & 49.5 & 47.0 & 62.5 & 62.5 & 66.6 & 76.7 & 44.7 & 53.1 & 64.9 & 67.5 & 55.6\\
& CAZO & 52.7 & 54.0 & 54.4 & 51.2 & 54.1 & 58.3 & 58.3 & 65.9 & 66.1 & 68.5 & 78.3 & 48.2 & 67.1 & 70.7 & 70.5 & \textbf{61.2} \\
\hline
\end{tabular}
}

\end{table*}

\subsection{Ablation Study}
\begin{figure*}[ht] 
    \centering 
    \begin{subfigure}{0.32\textwidth} 
        \centering
        \includegraphics[width=\textwidth]{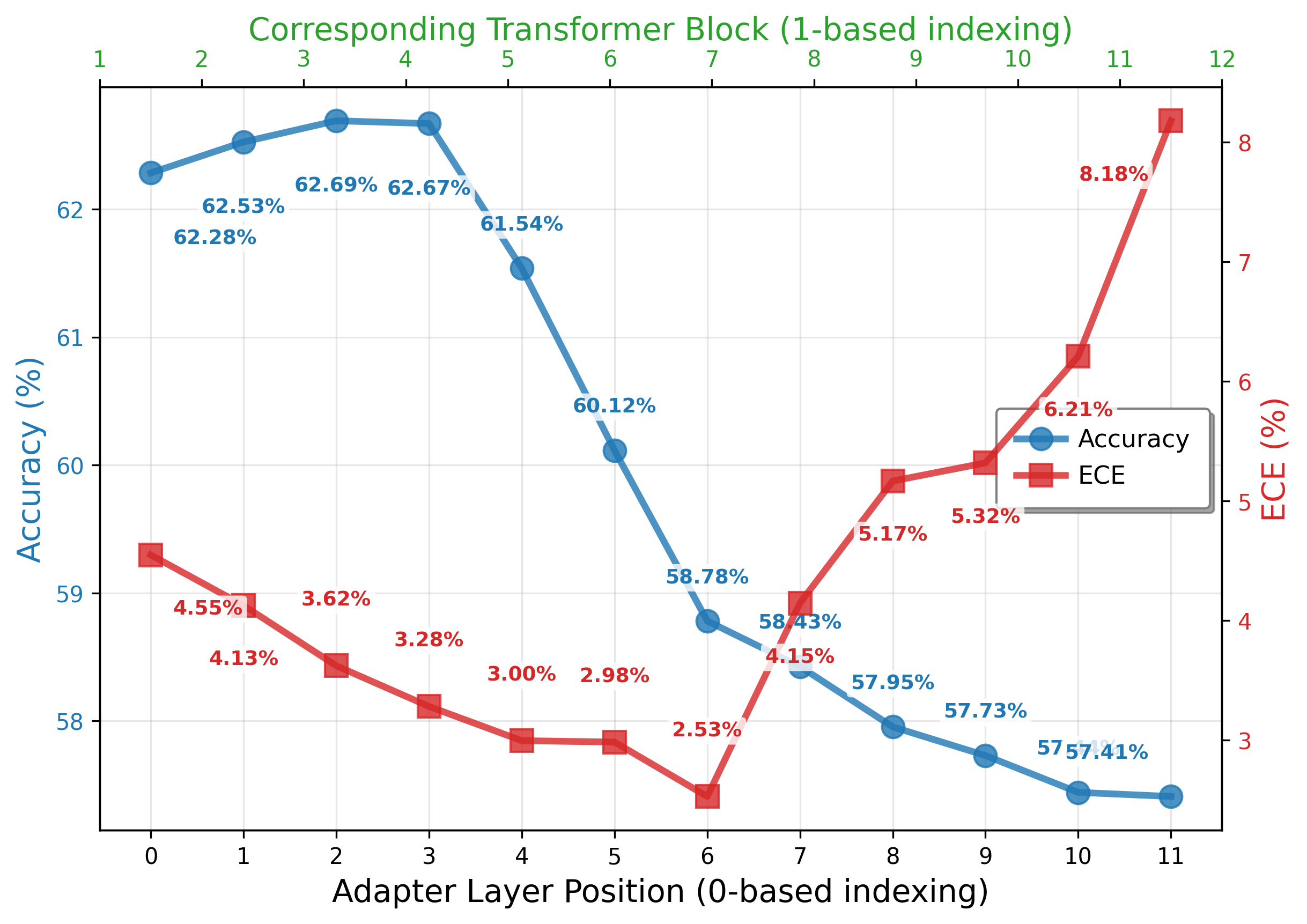} 
        \captionsetup{font=small,labelfont=bf}
        \caption{Effect of adapter layer position}
        \label{fig:diff_layer}
    \end{subfigure}
    \hfill 
    \begin{subfigure}{0.33\textwidth}
        \centering
        \includegraphics[width=\textwidth]{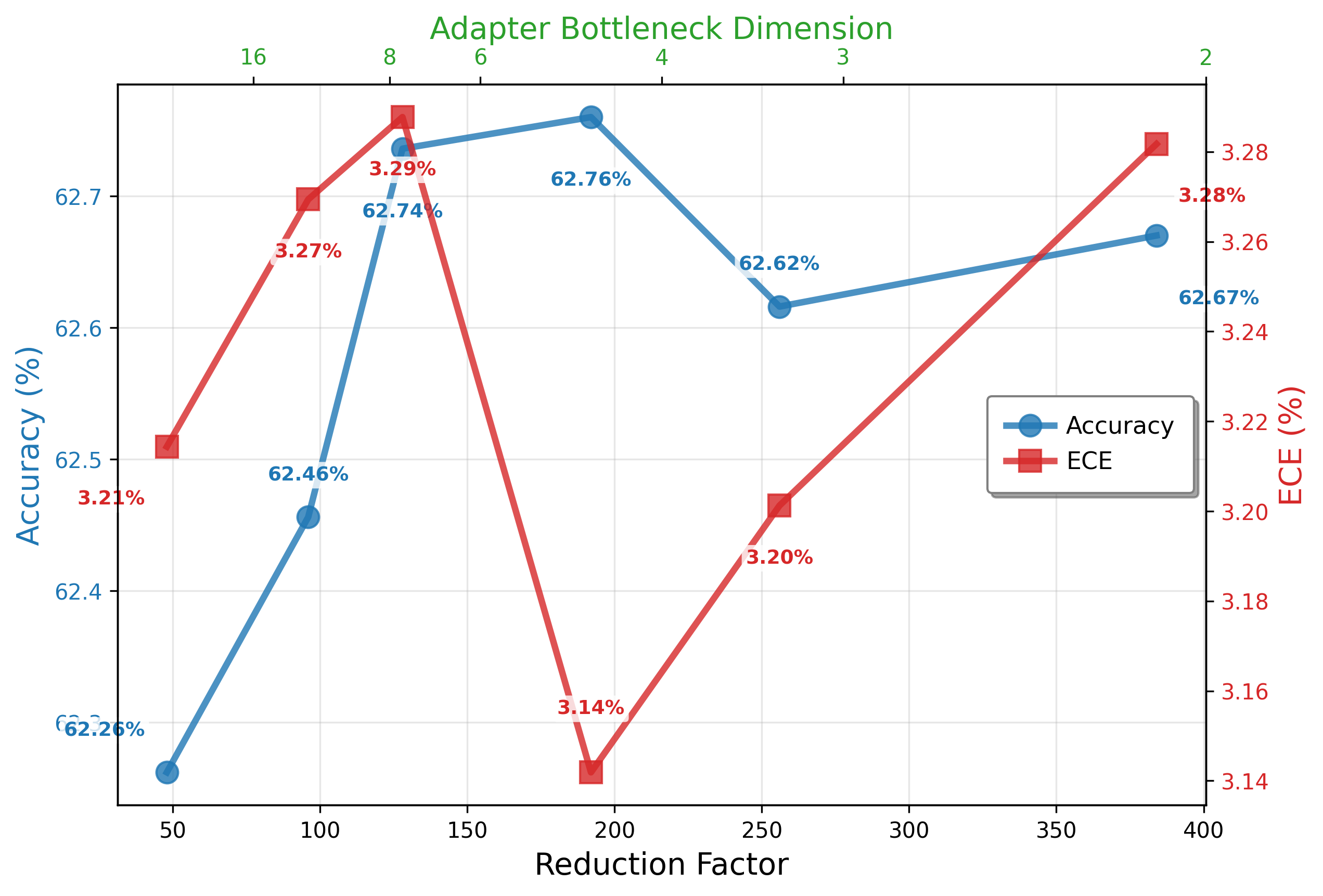}
        \captionsetup{font=small,labelfont=bf}
        \caption{Effect of adapter reduction factor}
        \label{fig:diff_ratio}
    \end{subfigure}
    \hfill
    \begin{subfigure}{0.32\textwidth}
        \centering
        \includegraphics[width=\textwidth]{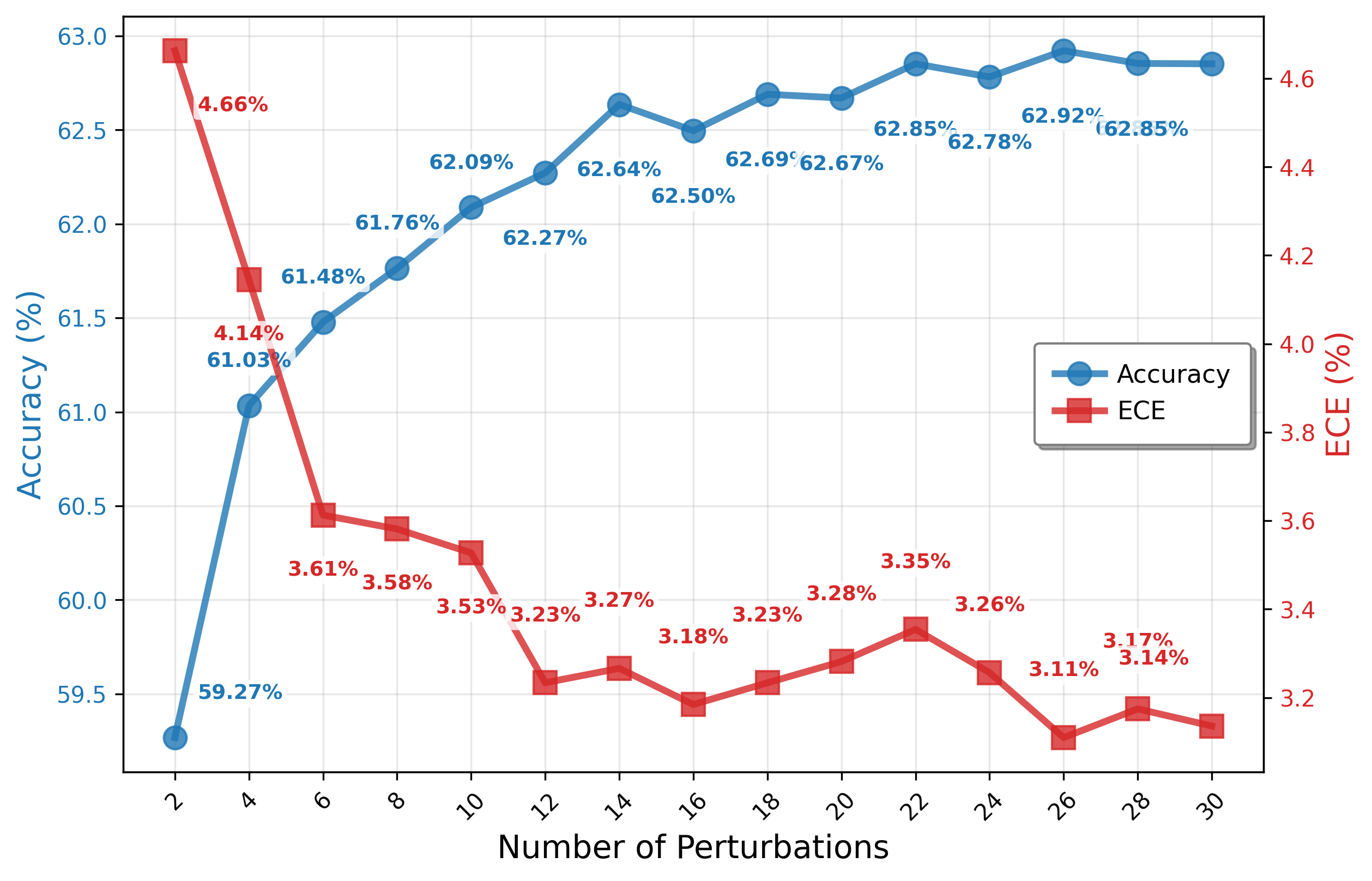}
        \captionsetup{font=small,labelfont=bf}
        \caption{Effect of perturbation number}
        \label{fig:perturbations}
    \end{subfigure}
    
    \caption{Parameter sensitivity analysis of CAZO on ImageNet-C (Gaussian noise, severity-5) with ViT-Base/16}
    \label{fig:ablation_sensitivity}
\vspace{-10pt}
\end{figure*}

\textbf{Adapter layer position.} 
Since CAZO updates only lightweight adapter modules for efficiency, we first study the effect of inserting the adapter into different layers of the 12-layer ViT-B/16 encoder. Following common practice \cite{chen2022adaptformer,pfeiffer2020adapterhub}, we apply the adapter to a single layer and report accuracy and ECE in Fig.~\ref{fig:diff_layer}. Results show that accuracy peaks at early layers, particularly layer 3, suggesting that low-level features are more critical for fast domain alignment. In contrast, later layers yield poorer results, likely due to limited adaptability of semantically stable representations. So we use layer 3 as the default location in all experiments.


\textbf{Adapter down-sampling ratio.} 
We further investigate the impact of the adapter's down-sampling ratio, which directly controls the number of trainable parameters. As shown in Fig.~\ref{fig:diff_ratio}, overall, the accuracy improvements are relatively marginal compared to the exponential increase in parameter count. Specifically, smaller ratios (e.g., 96, 128) introduce significantly more trainable parameters, yet offer no performance gains and can even degrade accuracy in some cases. The best results are observed at a ratio of 384, which achieves a favorable balance between adaptation effectiveness and parameter efficiency. 
These findings indicate that performance gains are not simply due to increased adaptation capacity, as larger adapters even hurt performance in some cases. 


\textbf{Effect of pertrubation strategy.} 
We evaluate how the number of perturbation directions \(k\) in ZO gradient estimation affects model performance. As shown in Fig.~\ref{fig:perturbations}, increasing \(k\) from 2 to 6 brings a significant accuracy gain (from 59.3\% to 62.1\%) and a sharp drop in ECE. Beyond that, both metrics exhibit marginal improvements and tend to stabilize. However, the number of perturbations is linearly correlated with the number of forward passes, which directly impacts runtime. Considering this trade-off, we adopt 20 as the default perturbation number in CAZO to ensure strong performance while keeping the adaptation time practical.

We also evaluate single-point perturbation to further reduce forward passes and runtime. However, we observe that this significantly degrades performance. As a result, we adopt symmetric multi-point perturbation in our ZO-based method implementation. More details of this comparison are provided in the supplemental material.

\textbf{Effect of $\nu$ in EMA.}
The smoothing factor \(\nu\) in \eqref{eq:EMA_hessian} controls the update rate of the diagonal Hessian estimate. We evaluate how the EMA update factor \(\nu\) affects adaptation performance. We find that moderate values (e.g., \(\nu=0.8\)) achieve the best accuracy (69.0\%) and calibration, while extreme values (e.g., \(\nu=1.0\)) lead to instability. More detailed results are provided in supplemental material.


%% file: sec/7_conclusion.tex
\section{Conclusion}
\label{sec:conclusion}

We proposed CAZO, a curvature-aware ZO optimization method for memory-efficient TTA. Motivated by the observation that the Hessian during TTA admits a low-dimensional principal subspace that is both prominent and slowly varying, CAZO leverages approximate curvature information to reduce the variance of ZO gradient estimation. 
Specifically, we instantiate CAZO by maintaining a sliding average of the diagonal covariance proxy to guide perturbation sampling. Extensive experiments on ImageNet-C, ImageNet-R, ImageNet-V2, and ImageNet-Sketch demonstrate that CAZO consistently outperforms both BP-free and BP-based TTA baselines, while substantially reducing memory overhead compared to BP-based methods.

%% file: sec/X_suppl.tex
\clearpage
\setcounter{page}{1}
\maketitlesupplementary

\setcounter{theorem}{0}
\setcounter{lemma}{0}

\section{Convergence Analysis of CAZO}
\label{sec:Suppl_convergence}

Given a loss function $\mathcal{L}(x;\theta)$, where $x \in \mathbb{R}^n$ is the data and $\theta \in \mathbb{R}^d$ is the parameter, then the definition of the gradient of CAZO is

\begin{equation}
    \begin{split}
        \label{eq:grad_estimator}
        \widetilde{g}(x_t;\theta_t) &= \frac{1}{k} \sum_{i=1}^k \frac{\mathcal{L}(x_t;\theta_t + \epsilon u_i) - \mathcal{L}(x_t;\theta_t - \epsilon u_i)}{2 \epsilon} u_i, \\
        &\quad \text{with } u_i \sim \mathcal{N}(0, \widetilde{H}^{-1}_t)
    \end{split}
\end{equation}

We further denote $\nabla f(\theta) = \mathbb{E}_{x}[\nabla f(x;\theta)].$
\begin{equation}
    \begin{split}
        D_t = (1-\nu)D_{t-1}+\nu \widetilde{g}^2(\theta_{t-1}), \\
        \widetilde{H}_t = \mathrm{diag}\left(\frac{D_t}{1-(1-\nu)^t}\right),
    \end{split}
\end{equation}

\begin{assumption*}[Restatement of L-smoothness]
    Assume the loss function $\mathcal{L}(x;\theta)$ is $L$-smooth respect to parameter $\theta$.
\end{assumption*}

\begin{assumption*}[Restatement of Data variance]
    Assume the data variance satisfies $\mathbb{E}_{x}[\| \nabla f(x;\theta) - \nabla f(\theta) \|] \leq \sigma^2$.
\end{assumption*}

\begin{assumption*}[Restatement of Value range of $\widetilde{H}^{-1}_t$]
    Assume for $t \geq 0$, the element in $\widetilde{H}^{-1}_t$ are in range $[\beta_l, \beta_u]$ where $0 < \beta_l \leq \beta_u \leq \infty$.
\end{assumption*}

\begin{lemma}[Estimation error of the zeroth order gradient estimator]
    Given a loss function $\mathcal{L}(\theta_t)$ with parameter $\theta_t$ satisfies Assumption \ref{assum:L_smooth}, and a data $x_t$ sampled from $\mathcal{D}$, a gradient estimator and variance matrix estimator in Eq.~\eqref{eq:grad_estimator} and Eq.~\eqref{eq:EMA_hessian} satisfies Assumption \ref{assum:data_var} and \ref{assum:hessian_value_range}, the following relation holds
    \begin{equation*}
        \begin{aligned}
            \mathbb{E}_{u_i,x_t}\left[ \widetilde{g}(x_t;\theta_t) \right] & = \widetilde{H}^{-1}_t \nabla \mathcal{L}(\theta_t) + \mathcal{O}(\epsilon)\\
            \mathbb{E}_{x_t,u_i}\left[ \| \widetilde{g}(x_t;\theta_t) \|^2 \right] 
            &\leq 2d(d+2)\beta_u \left(\| \nabla \mathcal{L}(\theta_t) \|^2 + \sigma^2 \right) \\
            &\quad + \mathcal{O}(\epsilon^2).
        \end{aligned}
    \end{equation*}
\end{lemma}

\begin{proof}
    By the mean value theorem, we have
        \begin{align*}
            \mathcal{L}(x_t;\theta_t + \epsilon u_i) 
            &= \mathcal{L}(x_t;\theta_t) + \epsilon \nabla \mathcal{L}(x_t;\theta_t)^{\top} u_i \\
            &\quad + \frac{\epsilon^2}{2} u_i^{\top} \nabla^2 \mathcal{L}(x_t;\xi^1_t) u_i, \\
            \mathcal{L}(x_t;\theta_t - \epsilon u_i) 
            &= \mathcal{L}(x_t;\theta_t) - \epsilon \nabla \mathcal{L}(x_t;\theta_t)^{\top} u_i \\
            &\quad + \frac{\epsilon^2}{2} u_i^{\top} \nabla^2 \mathcal{L}(x_t;\xi^2_t) u_i.
        \end{align*}
    So we then have
    \begin{equation}
        \begin{aligned}
            & \frac{\mathcal{L}(x_t;\theta_t + \epsilon u_i) - \mathcal{L}(x_t;\theta_t - \epsilon u_i)}{2 \epsilon} u_i \\
            = & \left(\nabla \mathcal{L}(x_t;\theta_t)^{\top} u_i\right) u_i + \frac{\epsilon}{4} \left( u_i^{\top} \nabla^2 \mathcal{L}(x_t;\xi^1_t) u_i \right. \\
            & \left. - u_i^{\top} \nabla^2 \mathcal{L}(x_t;\xi^2_t) u_i\right)  u_i,
        \end{aligned}
    \end{equation}
    where $\xi^1_t = \lambda_1 \theta_t + (1-\lambda_1) (\theta_t + \epsilon u_i)$, $\xi^2_t = \lambda_2 \theta_t + (1-\lambda_2) (\theta_t - \epsilon u_i)$ and $\lambda_1, \lambda_2 \in (0,1)$. We further denote 
    \begin{equation}
        \begin{split}
            A_i &= \left(\nabla \mathcal{L}(x_t;\theta_t)^{\top} u_i\right) u_i, \quad \\
            B_i &= \left( u_i^{\top} \nabla^2 \mathcal{L}(x_t;\xi^1_t) u_i - u_i^{\top} \nabla^2 \mathcal{L}(x_t;\xi^2_t) u_i\right)  u_i.
        \end{split}
    \end{equation}
    Then we have
    \begin{equation}
        \mathbb{E}_{u_i}[A_i] = \mathbb{E}_{u_i}\left[ u_i u_i^{\top} \right] \nabla \mathcal{L}(x_t;\theta_t) = \widetilde{H}^{-1}_t \nabla \mathcal{L}(x_t;\theta_t),
    \end{equation}
    \begin{equation}
        \begin{aligned}
            \mathbb{E}_{u_i}[\| A_i \|^2] 
            &= \mathbb{E}_{u_i}\left[ \left(\nabla \mathcal{L}(x_t;\theta_t)^{\top} u_i\right)^2 \| u_i \|^2\right] \\
            &\leq \| \nabla \mathcal{L}(x_t;\theta_t) \|^2 \mathbb{E}_{u_i}[\| u_i \|^4]
        \end{aligned}
        \label{eq:expectation_A_squared_extended}
    \end{equation}
    and because of Assumption \ref{assum:L_smooth}
    \begin{equation}
        \begin{aligned}
            u_i^{\top} \nabla^2 \mathcal{L}(x_t;&\xi^1_t) u_i - u_i^{\top} \nabla^2 \mathcal{L}(x_t;\xi^2_t) u_i 
            \leq 2L \| u_i\|^2 \\
            &\Rightarrow \mathbb{E}_{u_i}[B_i] \leq 2L \mathbb{E}_{u_i}[\| u_i\|^3]
        \end{aligned}
    \end{equation}
    \begin{equation}
        \begin{aligned}
            \mathbb{E}_{u_i}[\| B_i \|^2] \leq 4L^2 \mathbb{E}_{u_i}[\| u_i \|^6]
        \end{aligned}
    \end{equation}
    Now we start to estimate the conclusion
    \begin{equation}
        \begin{aligned}
            \mathbb{E}_{u_i}\left[ \widetilde{g}(x_t;\theta_t) \right]  
            &= \frac{1}{k} \sum_{i=1}^k \left( \mathbb{E}_{u_i}[A_i] +  \frac{\epsilon}{4}\mathbb{E}_{u_i}[B_i]\right) \\
            &= \mathbb{E}_{u_i}[A_i] +  \frac{\epsilon}{4}\mathbb{E}_{u_i}[B_i] \\
            &= \widetilde{H}^{-1}_t \nabla \mathcal{L}(x_t;\theta_t) + \frac{\epsilon}{4} \mathbb{E}_{u_i}[B_i],
        \end{aligned}
        \label{eq:expectation_gradient}
    \end{equation}
    Then take expectation of $x_t$ on both size of the inequality, we have
    \begin{equation}
        \mathbb{E}_{x_t,u_i}\left[ \widetilde{g}(x_t;\theta_t) \right] = \widetilde{H}^{-1}_t \nabla \mathcal{L}(\theta_t) + \frac{\epsilon}{4} \mathbb{E}_{u_i}[B_i].
    \end{equation}
    By Jensen's inequality and Cauchy-Schwarz inequality, we have
    \begin{equation}
        \begin{aligned}
            \| \widetilde{g}(x_t;\theta_t) \|^2  
            &= \left\| \frac{1}{k} \sum_{i=1}^k \left( A_i + \frac{\epsilon}{4}B_i \right) \right\|^2 \\
            &\leq \frac{1}{k} \sum_{i=1}^k \left\| A_i + \frac{\epsilon}{4}B_i \right\|^2 \\
            &\leq \frac{2}{k} \sum_{i=1}^k \left( \| A_i \|^2 + \left\| \frac{\epsilon}{4}B_i \right\|^2 \right),
        \end{aligned}
    \end{equation}
    and because $u_i$ is I.I.D samples, we have
    \begin{equation}
        \begin{split}
            \mathbb{E}_{u_i} \left[ \| \widetilde{g}(x_t;\theta_t) \|^2 \right] 
            &\leq \frac{2}{k} \sum_{i=1}^k \left( \mathbb{E}_{u_i}[\| A_i \|^2] + \mathbb{E}_{u_i}[\| \tfrac{\epsilon}{4}B_i \|^2] \right) \\
            &= 2\mathbb{E}_{u_i}[\| A_i \|^2] + 2\mathbb{E}_{u_i}[\|\tfrac{\epsilon}{4} B_i \|^2] \\
            &\leq 2\| \nabla \mathcal{L}(x_t;\theta_t) \|^2 \mathbb{E}_{u_i}[\| u_i \|^4] \\
            &\quad + \frac{\epsilon^2 L^2}{2} \mathbb{E}_{u_i}[\| u_i \|^6].
        \end{split}
    \end{equation}
    Take the expectation of $x_t$ on both size of the inequality, and using the assumption \ref{assum:data_var}, we have
    \begin{equation}
        \begin{aligned}
            \mathbb{E}_{x_t,u_i} \left[ \| \widetilde{g}(x_t;\theta_t) \|^2 \right] 
            &\leq 2 \mathbb{E}_{x_t}[\| \nabla \mathcal{L}(x_t;\theta_t) \|^2] \mathbb{E}_{u_i}[\| u_i \|^4] \\
            &\quad + \frac{\epsilon^2 L^2}{2} \mathbb{E}_{u_i}[\| u_i \|^6] \\
            &\leq 2 \| \nabla \mathcal{L}(\theta_t) \|^2 \mathbb{E}_{u_i}[\| u_i \|^4] \\
            &\quad + 2\sigma^2 \mathbb{E}_{u_i}[\| u_i \|^4] \\
            &\quad + \frac{\epsilon^2 L^2}{2} \mathbb{E}_{u_i}[\| u_i \|^6]
        \end{aligned}
    \end{equation}
    Now we start to estimate the upper bound of the moments of $u_i$. Because $\widetilde{H}^{-1}_t$ is a diagonal matrix and Assumption \ref{assum:hessian_value_range}, we have
    \begin{equation}
        \begin{aligned}
            \mathbb{E}_{u_i}[\| u_i \|^4] & = (\operatorname{tr}(\widetilde{H}^{-1}_t))^2 + 2 \operatorname{tr}(\widetilde{H}^{-1}_t) \\
            &\leq d^2 \beta_u + 2d \beta_u \\
            &= d(d+2)\beta_u
        \end{aligned}
    \end{equation}
    So we proved the result.
\end{proof}

\begin{theorem}[Convergence rate of CAZO]
    Given a loss function $\mathcal{L}$ satisfies Assumption \ref{assum:L_smooth}, a gradient estimator and variance matrix estimator in Eq.~\eqref{eq:grad_estimator} and Eq.~\eqref{eq:EMA_hessian} satisfies Assumption \ref{assum:data_var} and \ref{assum:hessian_value_range}, with training time $T$, learning rate $\eta = \frac{\beta_l}{2Ld(d+2) \beta_u \sqrt{T}}$, the convergence rate of the CAZO is 
    \begin{equation}
        \begin{aligned}
            \frac{1}{T} \sum_{t=1}^T \mathbb{E}\left[ \| \nabla \mathcal{L}(\theta_t) \|^2\right] & \leq \frac{4Ld(d+2)\beta_u\left(\mathcal{L}(\theta_{t}) - \mathcal{L}(\theta^*) \right)}{\beta_l^2 \left( \sqrt{T} - 1\right)}\\
            &\quad + \frac{\sigma^2}{\sqrt{T} - 1} + \mathcal{O}(\epsilon^2).
        \end{aligned}
    \end{equation}
\end{theorem}

\begin{proof}
    Because of Assumption \ref{assum:L_smooth}, we have $f(\theta) = \mathbb{E}_{x}[f(x;\theta)]$ is also a L-smooth function. So $f(\theta)$ satisfies the following inequality
    \begin{equation}
        \begin{aligned}
            \mathcal{L}(\theta_{t+1}) \leq& \mathcal{L}(\theta_{t}) - \eta  \nabla \mathcal{L}(\theta_t)^{\top} \widetilde{g}(x_t;\theta_t)\\
            &+ \frac{L \eta^2}{2}\|\widetilde{g}(x_t;\theta_t) \|^2
        \end{aligned}
    \end{equation}
    Using the result from Lemma \ref{lemma:mean_and_variance}, take expectation on $u_i$ and $x_t$ and choose a suitable $\epsilon$, we can have 
    \begin{equation}
    \begin{aligned}
        \mathbb{E}_{x_t}[\mathcal{L}(\theta_{t+1}) | \theta_t] 
        &\leq \mathcal{L}(\theta_{t}) - \eta \nabla \mathcal{L}(\theta_t)^{\top} \widetilde{H}^{-1}_t \nabla \mathcal{L}(\theta_t) \\
        &\quad + \eta \mathcal{O}(\epsilon \| \nabla \mathcal{L}(\theta_t) \|) \\
        &\quad + L \eta^2d(d+2)\beta_u \left(\| \nabla \mathcal{L}(\theta_t) \|^2 + \sigma^2 \right) \\
        &\quad + \mathcal{O}(\epsilon^2) \\
        &\leq \mathcal{L}(\theta_{t}) - \frac{\eta}{2} \| \nabla \mathcal{L}(\theta_t) \|^2_{\widetilde{H}^{-1}_t} \\
        &\quad+ L \eta^2d(d+2)\beta_u \| \nabla \mathcal{L}(\theta_t) \|^2 \\
        &\quad + L \eta^2d(d+2)\beta_u \sigma^2 + \mathcal{O}(\epsilon^2) \\
        &\leq \mathcal{L}(\theta_{t}) - \frac{\eta \beta_l}{2} \| \nabla \mathcal{L}(\theta_t) \|^2 \\
        &\quad + L \eta^2d(d+2)\beta_u \| \nabla \mathcal{L}(\theta_t) \|^2 \\
        &\quad + L \eta^2d(d+2)\beta_u \sigma^2 + \mathcal{O}(\epsilon^2) \\
        &= \mathcal{L}(\theta_{t}) + L \eta^2d(d+2)\beta_u \sigma^2 \\
        &\quad- \left( \frac{\eta \beta_l}{2} - L \eta^2d(d+2)\beta_u \right)
            \| \nabla \mathcal{L}(\theta_t) \|^2 \\
        &\quad + \mathcal{O}(\epsilon^2).
    \end{aligned}
\end{equation}
    The last inequality holds because Assumption \ref{assum:hessian_value_range} and $\widetilde{H}^{-1}_t$ is a diagonal matrix. Then we have
    \begin{equation}
        \begin{gathered}
            \left( \frac{\eta \beta_l}{2} - L \eta^2d(d+2)\beta_u \right) \| \nabla \mathcal{L}(\theta_t) \|^2 \\
            \leq \mathcal{L}(\theta_{t}) - \mathbb{E}_{x_t}[\mathcal{L}(\theta_{t+1}) | \theta_t] \\
            \quad + L \eta^2d(d+2)\beta_u \sigma^2 + \mathcal{O}(\epsilon^2).
        \end{gathered}
    \end{equation}
    Sum the term from $t = 1 $ to $t = T$, take expectation, and using the telescoping sum, we have
    \begin{equation}
        \begin{gathered}
            \mathbb{E}\left[\sum_{t=1}^T \left( \frac{\eta \beta_l}{2} - L \eta^2d(d+2)\beta_u  \right)\| \nabla \mathcal{L}(\theta_t) \|^2\right] \\
            \leq \mathcal{L}(\theta_{0}) - \mathbb{E}[\mathcal{L}(\theta_{T})] \\
            + T L \eta^2d(d+2)\beta_u \sigma^2 + \mathcal{O}(T\epsilon^2).
        \end{gathered}
    \end{equation}
    Because $\eta = \frac{\beta_l}{2Ld(d+2) \beta_u \sqrt{T}}$, we have
    \begin{equation}
         \frac{\eta \beta_l}{2} - L \eta^2d(d+2)\beta_u = \frac{\beta_l^2 \left( \sqrt{T} - 1\right)}{4Ld(d+2)\beta_u T}, 
    \end{equation}
    and
    \begin{equation}
        T L \eta^2d(d+2)\beta_u \sigma^2 = \frac{\beta_l^2 \sigma^2}{4Ld(d+2)\beta_u}
    \end{equation}
    Because $\mathbb{E}[\mathcal{L}(\theta_{T})] \leq \mathcal{L}(\theta^*)$ where $\theta^*$ is the global minimum and we divide both side by $\left( \frac{\eta \beta_l}{2} - L \eta^2d(d+2)\beta_u  \right)T$, we can obtain the final resultas follow
    \begin{equation}
        \begin{aligned}
            \frac{1}{T} \sum_{t=1}^T \mathbb{E}\left[ \| \nabla \mathcal{L}(\theta_t) \|^2\right] & \leq \frac{4Ld(d+2)\beta_u\left(\mathcal{L}(\theta_{0}) - \mathcal{L}(\theta^*) \right)}{\beta_l^2 \left( \sqrt{T} - 1\right)}\\
            &\quad + \frac{\sigma^2}{\sqrt{T} - 1} + \mathcal{O}(\epsilon^2).
        \end{aligned}
    \end{equation}
    The proof is finished.

    \begin{equation}
    \begin{gathered}
        \frac{1}{T}\sum_{t=1}^T\mathbb{E}\left[\|\nabla\mathcal{L}(\theta_t)\|^2\right] \\
        \leq\mathcal{O}\left(\frac{d^2\beta_u}{\beta_l^2\left(\sqrt T-1\right)}+\frac{\sigma^2}{\sqrt T-1}\right)+\mathcal{O}(\mu^2).
    \end{gathered}
    \end{equation}
\end{proof}

\section{Experimental Settings}
\label{SM:experiment_settings}

\subsection{Dataset and Model}
We evaluate our method on ImageNet-C, a benchmark dataset for test-time adaptation \cite{hendrycks2018benchmarking}. ImageNet-C contains 15 common corruption types, each with five severity levels, yielding 75 corrupted versions of the ImageNet validation set. The corruptions include noise, blur, weather-related distortions, and other real-world degradations. We use ViT-Base/16 \cite{Dosovitskiy2021vit} as the source model. The ViT-Base model consists of 12 layers, with a hidden dimension of 768. The model is pretrained on the original ImageNet datasetv\cite{deng2009imagenet}, and we evaluate its performance in the test-time adaptation setting on ImageNet-C. 
Furthermore, we evaluate our approach on three domain-shifted benchmark datasets: ImageNet-R \cite{hendrycks2020augmix}, ImageNet-V2 \cite{recht2019imagenet}, and ImageNet-Sketch \cite{wang2019learning}. 

\textbf{ImageNet-V2} is a newly curated test dataset derived from the same underlying distribution as the original ImageNet. This dataset consists of three distinct test sets, each containing 10,000 images, and collectively spans 1,000 classes found in ImageNet. In line with prior methods for test-time augmentation (TTA) \cite{nado2020evaluating}, we specifically employ the Matched-Frequency subset of ImageNet-V2 for our evaluation. This subset is designed such that the images are sampled to align with the class frequency distributions of the original ImageNet validation dataset.

\textbf{ImageNet-R (Renditions)} provides a specialized benchmark for evaluating robustness against \textit{concept shifts}. It contains 30,000 images spanning 200 ImageNet classes, curated from diverse non-photorealistic renditions including artwork, cartoons, graffiti, sculptures, and video game renders. Crucially, all images are selected from samples misclassified by ResNet-50 models, emphasizing challenging semantic variations. The label space aligns with ImageNet-2012, enabling direct compatibility with standard evaluation pipelines. This dataset is particularly effective for testing model adaptability to abstract representations where texture and shape biases significantly impact performance \cite{hendrycks2021many}.

\textbf{ImageNet-Sketch} focuses on \textit{structural generalization} by exclusively using hand-drawn sketches of ImageNet objects\cite{wang2019learning}. Each sketch replicates the original 1,000-class structure, providing a domain-shifted testbed devoid of photographic textures. Models relying heavily on texture cues exhibit significant performance drops on this dataset, making it a critical benchmark for evaluating shape-based reasoning capabilities. Its construction mirrors the ImageNet validation set in class distribution and scale, ensuring comparable statistical rigor.


\subsection{Baseline Methods}
We compare CAZO with two categories of methods: non-backpropagation-based methods and backpropagation-based methods. The non-backpropagation methods include:
LAME \cite{boudiaf2022parameter}, a post-training adaptation technique that refines the model's output probabilities; 
T3A \cite{iwasawa2021test}, which adjusts the model's linear classifier during inference; 
and FOA \cite{niu2024test}, which employs a covariance matrix adaptation evolution strategy to update prompts. 
In addition, we include ZOA baseline~\cite{deng2025test} that estimates gradients via two-point random perturbations without backpropagation; each adaptation step uses multiple forward-only evaluations to form a ZO gradient estimate. 
Backpropagation-based methods include TENT \cite{wang2021tent}, which minimizes entropy to fine-tune normalization layer parameters; CoTTA \cite{wang2022continual}, which combines knowledge distillation with data augmentation; and SAR \cite{niu2023towards}, which selects reliable samples to stabilize predictions. 
EATA~\cite{niu2022efficient} further improves entropy-minimization based TTA by filtering unreliable samples and adding a lightweight regularization to reduce forgetting, typically updating only BN-related parameters with BP under an entropy loss; this yields efficient, stable adaptation under streaming data.
DeYO~\cite{lee2024entropy} argues that pure entropy minimization can be deceptive on domain-shifted targets; it disentangles/regularizes latent factors to curb degeneration, combining entropy terms with factor-aware constraints under BP-based updates to improve robustness in continual settings.
RoTTA~\cite{yuan2023robust} focuses on temporally varying test streams, using teacher--student consistency, distribution-aware augmentation and memory mechanisms to remain stable across dynamic shifts; adaptation proceeds with BP while controlling drift.
LCoTTA~\cite{duan2025lifelong} identifies \emph{entropy-deceptive (ED)} samples as the cause of degeneration in continual TTA, reveals that entropy-minimization gradients possess a low-dimensional principal subspace dominated by \emph{entropy-truthful (ET)} samples, and constrains weight updates by \emph{tracking} this principal subspace online and \emph{projecting} gradients into it. This subspace-projected BP update suppresses ED gradients, stabilizing long-horizon adaptation across many cycles.

Additionally, to demonstrate the effectiveness of the curvature-aware sampling approach in CAZO, we consider a ZO baseline algorithm using the standard RGE perturbation form \cite{nesterov2017random,duchi2015optimal} within the same framework of CAZO. This ZO baseline employs standard Gaussian distribution perturbations and estimates gradients through multiple double-point perturbations, which maintains the BP-free nature of our approach. 

\subsection{Implementation Details}
In our experiments, we set the batch size to 64 and the \textbf{learning rate} to 0.01 for all algorithms. We use the ViT-B/16 model with 12 transformer layers, integrating an adapter into an early layer for adaptation. The adapter employs a scaling factor of 384 for upsampling and downsampling, with additional tests on multiples of 768 to better align with the ViT model dimensions. The adapter is initialized using Kaiming initialization for downsampling and zero initialization for upsampling to ensure stable adaptation. The number of perturbations is set to 20 for all ZO-based method. The adapter scaling factor is 0.1, based on the configuration in AdaptFormer \cite{chen2022adaptformer}. For the ImageNet-C dataset, we evaluate on the most severe corruption level (level 5). And we set five \textbf{random-seed}: \(seed \in \{42, 2020, 2025, 1234, 888\}\). 
\begin{table*}[ht]
\centering
\caption{Experiment hyperparameters}
\label{tab:SM_hyperparameters}
\begin{tabular}{ll}
\toprule
\textbf{Hyperparameter} & \textbf{Value} \\
\midrule
Batch size & 64 \\
Learning rate & 0.01 \\
Vision Transformer model & ViT-B/16 (12 transformer layers) \\
Adapter integration location & Early layer(layer = 3) \\
Adapter downsampling scaling factor & Main: 384 \\
& Test variants: Multiples of 768 \\
Adapter initialization & Downsampling: Kaiming \\
& Upsampling: Zero \\
Adapter residual scaling factor & 0.1 \\
Number of perturbations (ZO-based methods) & 20 \\
ImageNet-C corruption level & 5 (most severe) \\
Random seeds & \{42, 2020, 2025, 1234, 888\} \\
\bottomrule
\end{tabular}%
\end{table*}

\subsection{Evaluation Metrics}
In our experiments, we evaluate the performance of our model using two primary metrics: classification accuracy (ACC) and Expected Calibration Error (ECE). ACC measures the proportion of correctly classified samples on the out-of-distribution (OOD) data, providing an indication of the model's robustness in handling data that differs from the training distribution. ECE quantifies the calibration of predicted probabilities, reflecting the difference between the predicted confidence and the actual accuracy of predictions \cite{naeini2015obtaining}. A lower ECE indicates better calibration, meaning the model's predicted probabilities align more closely with the true likelihood of correct predictions.

\section{More Experiment Details}

\subsection{Analysis of Hessian Matrix Properties in TTA}
\label{app:hessian_analysis}


\begin{figure*}[ht]
\centering
\begin{subfigure}[c]{0.45\linewidth}  
    \centering
    \includegraphics[width=\linewidth]{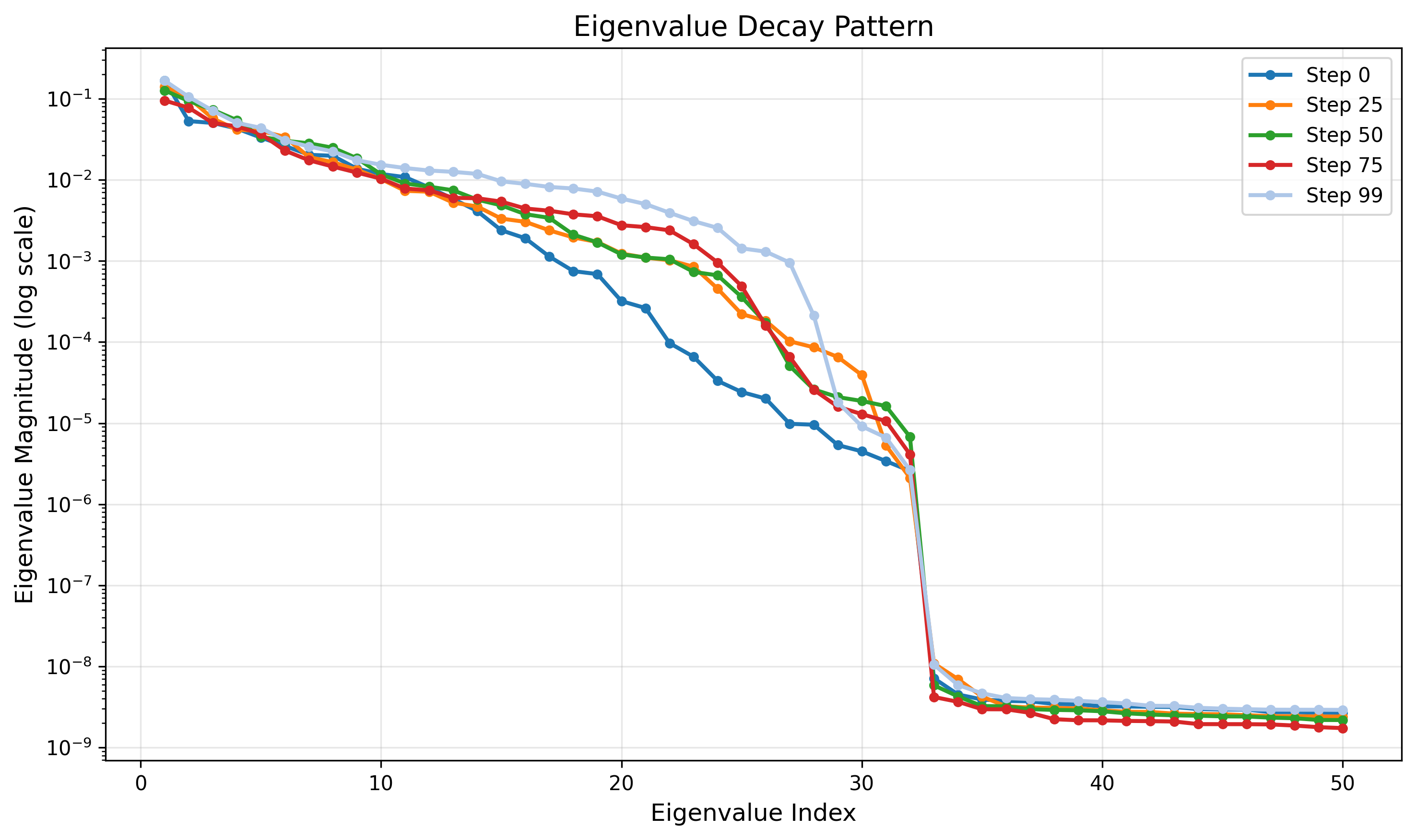}
    \caption{}
\end{subfigure}
\begin{subfigure}[c]{0.45\linewidth}
    \centering
    \includegraphics[width=\linewidth]{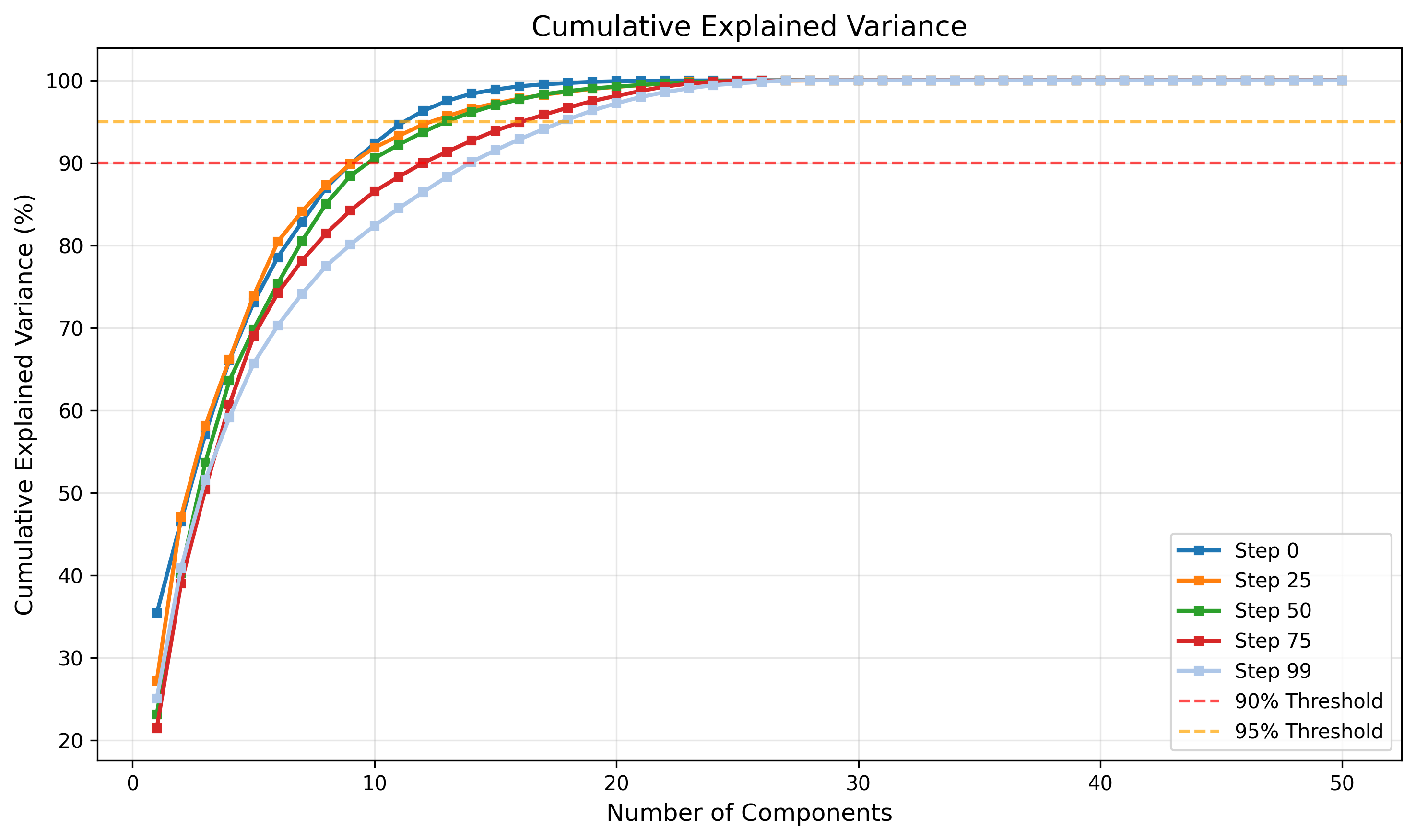}
    \caption{}
\end{subfigure}
\begin{subfigure}[c]{0.45\linewidth}
    \centering
    \includegraphics[width=\linewidth]{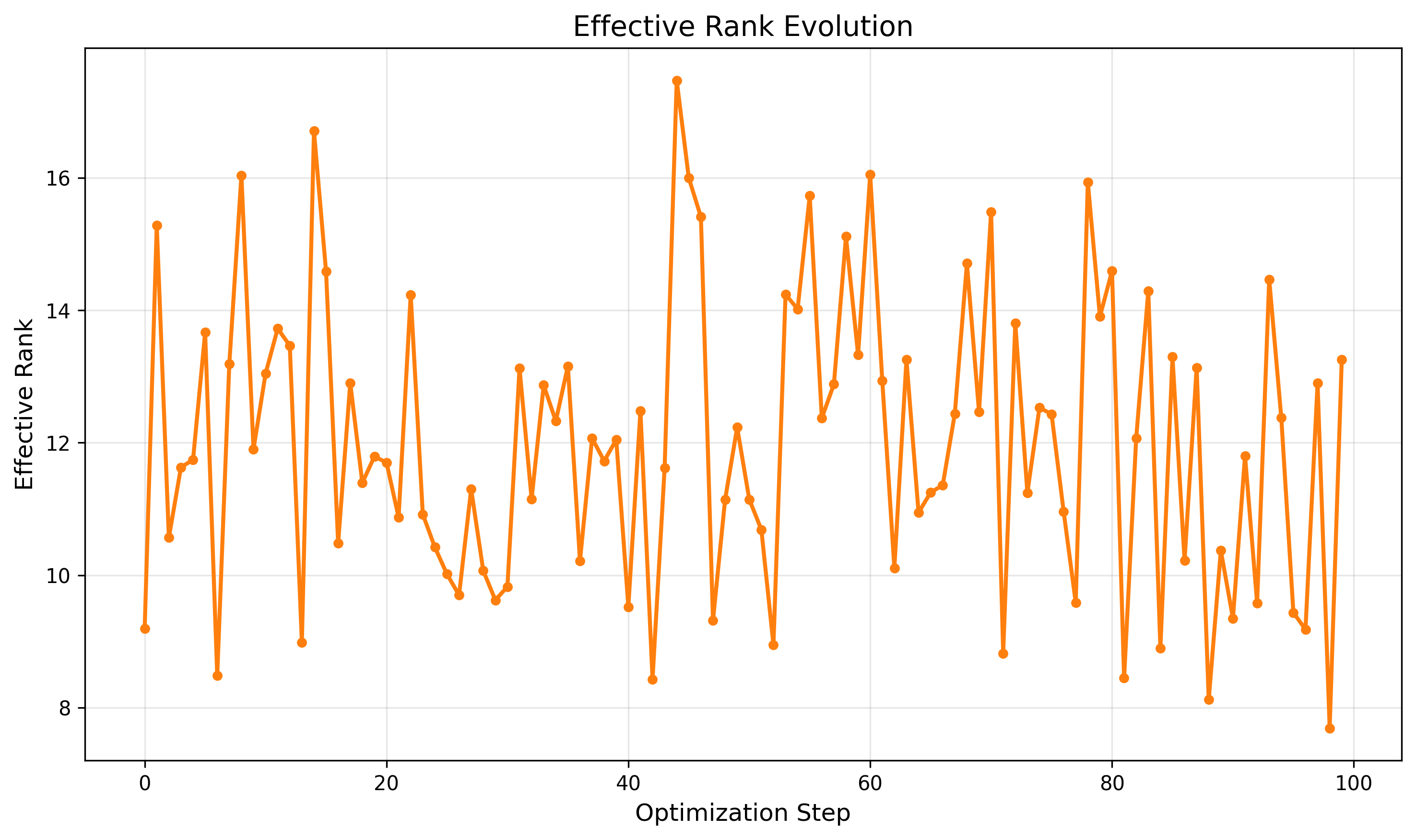}
    \caption{}
\end{subfigure}
\begin{subfigure}[c]{0.45\linewidth}
    \centering
    \includegraphics[width=\linewidth]{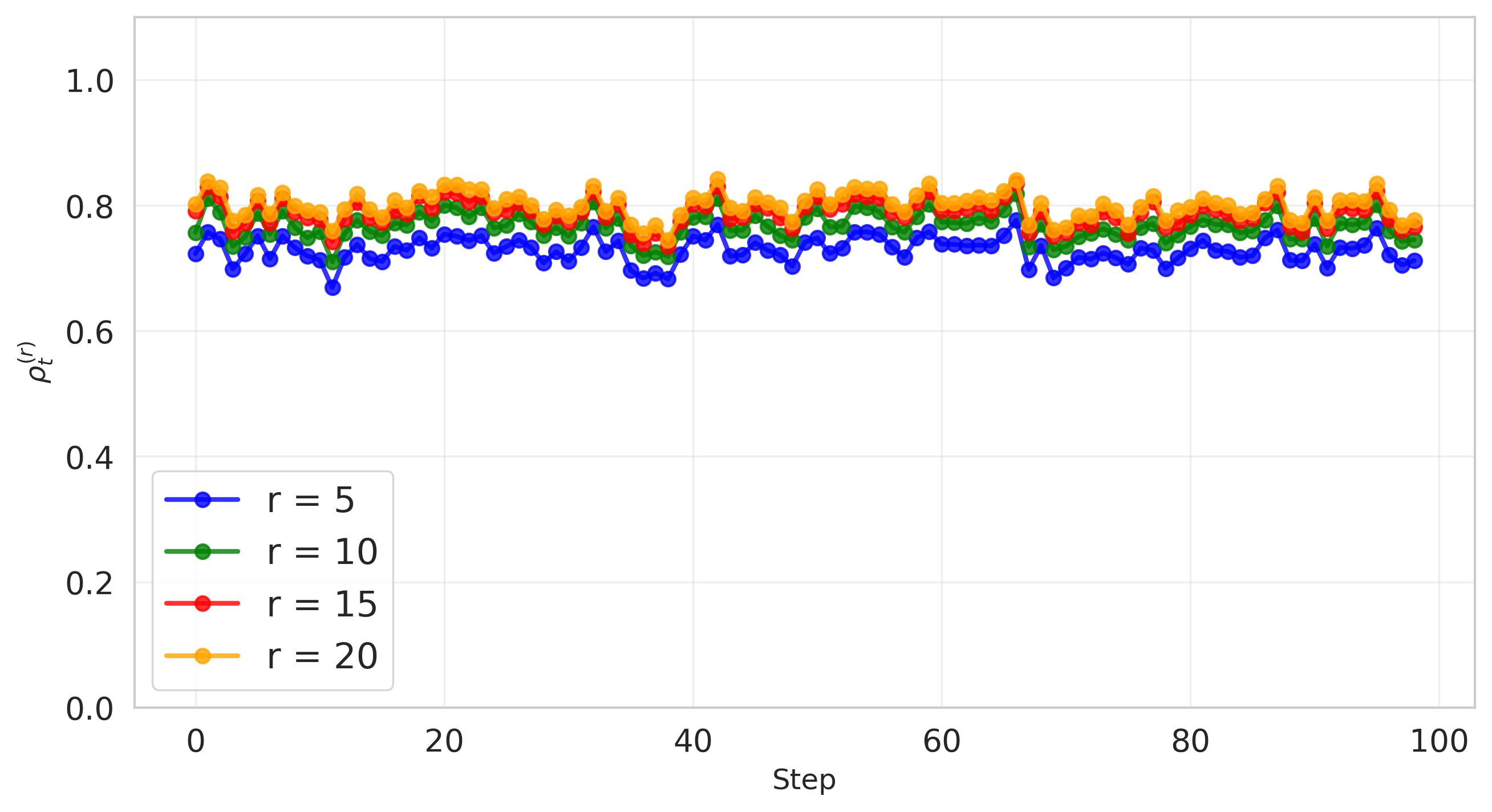}
    \caption{}
\end{subfigure}
\caption{Hessian eigenvalue distribution analysis during test-time adaptation process. Four complementary perspectives are presented: (a) Eigenvalue decay pattern across optimization steps; (b) Cumulative explained variance of principal components; (c) Evolution of effective rank; (d) Projection rate with batch size 200 and 100 steps}
\label{fig:hessian_analysis}
\end{figure*}

\paragraph{Motivation Experiment Description}
This pilot study investigates the structural properties of Hessian matrices during test-time adaptation (TTA). 
During experimentation, we employed a batch size of 32 and executed 100 optimization steps. At each step t, we recorded the Hessian matrix alongside critical quantities including weight updates, gradient statistics, and the effective rank derived from spectral analysis. The effective rank computation followed Shannon entropy principles:
$\textit{effective rank} = \exp\left( -\sum_{i} p_i \log p_i \right)$
where $p_i = {\lvert \lambda_i \rvert} \big/ {\sum_{j=1}^{n} \lvert \lambda_j \rvert}$ represents the normalized magnitude of the $i$-th eigenvalue $\lambda_i$ of the Hessian matrix. A lightweight single-layer adapter module was incorporated, with all parameters continuously updated through standard backpropagation throughout the test-time adaptation (TTA) process.
Through comprehensive eigenvalue spectrum analysis shown in Figure~\ref{fig:hessian_analysis}, we examine some critical aspects:
\begin{itemize}[leftmargin=*]
    \item \textbf{Eigenvalue Decay Pattern} (Fig.~\ref{fig:hessian_analysis}a): Demonstrates exponential magnitude decay ($10^{-1}$ to $10^{-9}$) across optimization steps. Notably, the first 10 eigenvalues contain $>99\%$ of total spectral energy, with decay slopes remaining consistent throughout adaptation.
    
    \item \textbf{Cumulative Explained Variance} (Fig.~\ref{fig:hessian_analysis}b): Confirms that fewer than 15 principal components capture $>90\%$ variance at all optimization stages (consistently above both 90\% and 95\% thresholds), indicating dimensionality saturation.
    
    
    \item \textbf{Effective Rank Evolution} (Fig.~\ref{fig:hessian_analysis}c): Quantifies stable low-dimensional structure, where effective rank plateaus at $\sim$12 after initial optimization phase.
\end{itemize}

\paragraph{Hessian Matrix Projection Ratio Analysis in TTA}
This section presents a detailed analysis of the projection ratio $\rho_t^{(r)}$, which quantifies how much the principal curvature directions of the Hessian are preserved across test-time adaptation steps. As shown in Figure~\ref{fig:hessian_analysis}, we investigate this ratio under different batch sizes, using a batch size of 512 for the primary experiments (Figure~\ref{fig:hessian_analysis}a-c), and a batch size of 200 for comparison (Figure~\ref{fig:hessian_analysis}d). The projection ratio for the larger batch size remains relatively stable throughout the adaptation process, reflecting a smoother and more consistent estimation of the Hessian's principal subspace. In contrast, with batch size 200, the projection ratio shows greater fluctuations, indicating higher instability in the Hessian estimation. This observation suggests that using larger batch sizes helps reduce the noise in curvature estimates, leading to more stable adaptation dynamics.

\begin{itemize}[leftmargin=*]
    \item \textbf{Projection Ratio with Batch Size 200 }(Fig.~\ref{fig:hessian_analysis}d): When using batch size 200, the projection ratio exhibits more significant fluctuations, suggesting that a larger batch size selection can better fit the Loss landscape, thereby providing a more powerful explanation of the slow changing nature of curvature. Therefore, batch size 512 is used as the main experiment in the main text.
\end{itemize}

\paragraph{Key Findings}
Three fundamental observations from this analysis:
\begin{itemize}[leftmargin=*]
    \item \textbf{Consistent Low-Rank Structure}: The persistent eigenvalue decay patterns across optimization steps (across steps 0, 25, 99) reveals intrinsic low-rank properties independent of adaptation progress.
    
    \item \textbf{Dimensional Inertia}: The stabilized effective rank ($r_{\text{eff}} \approx 12$) and variance concentration demonstrate fixed intrinsic dimensionality, despite continuous parameter updates.

    \item \textbf{Slow-Varying Curvature}: The Hessian matrix exhibits a slow-varying property over time, as evidenced by the stable projection ratio and effective rank. The principal curvature directions in the Hessian remain largely consistent across adaptation steps, suggesting that the underlying loss landscape does not undergo abrupt changes during the adaptation process.
\end{itemize}


\subsection{Entropy-Only Loss Supplement}
\label{app:entropy_only_loss}

We further evaluate CAZO under an \emph{entropy-only} objective, i.e.,
$L_{\text{ent}}$ without the feature-alignment term used in our default composite loss
$L_{\text{com}} = L_{\text{ent}} + L_{\text{align}}$.
This setting is known to be challenging for BP-free TTA, and we observe a substantial performance degradation for forward-only baselines.
Nevertheless, CAZO remains significantly more robust than FOA and vanilla ZO (RGE) under the same protocol.

\begin{table}[t]
\centering
\small
\caption{Entropy-only loss results under the same protocol (ImageNet-C, severity-5).}
\label{tab:entropy_only_loss}
\begin{tabular}{lccc}
\toprule
Method & FOA & ZO (RGE) & CAZO \\
\midrule
Acc. (\%, $\uparrow$) & 44.90 & 47.32 & 56.52 \\
\bottomrule
\end{tabular}
\end{table}

In addition, our Hessian observations are not loss-specific: we also ran the Hessian analysis using the composite loss $L_{\text{com}}$
and still observe a persistent low-rank spectrum as well as a slowly varying dominant subspace (Fig.~\ref{fig:hessian_composite_loss}).

\begin{figure}[htbp]
  \centering
  \begin{minipage}[t]{\linewidth}
    \centering
    \includegraphics[width=\linewidth]{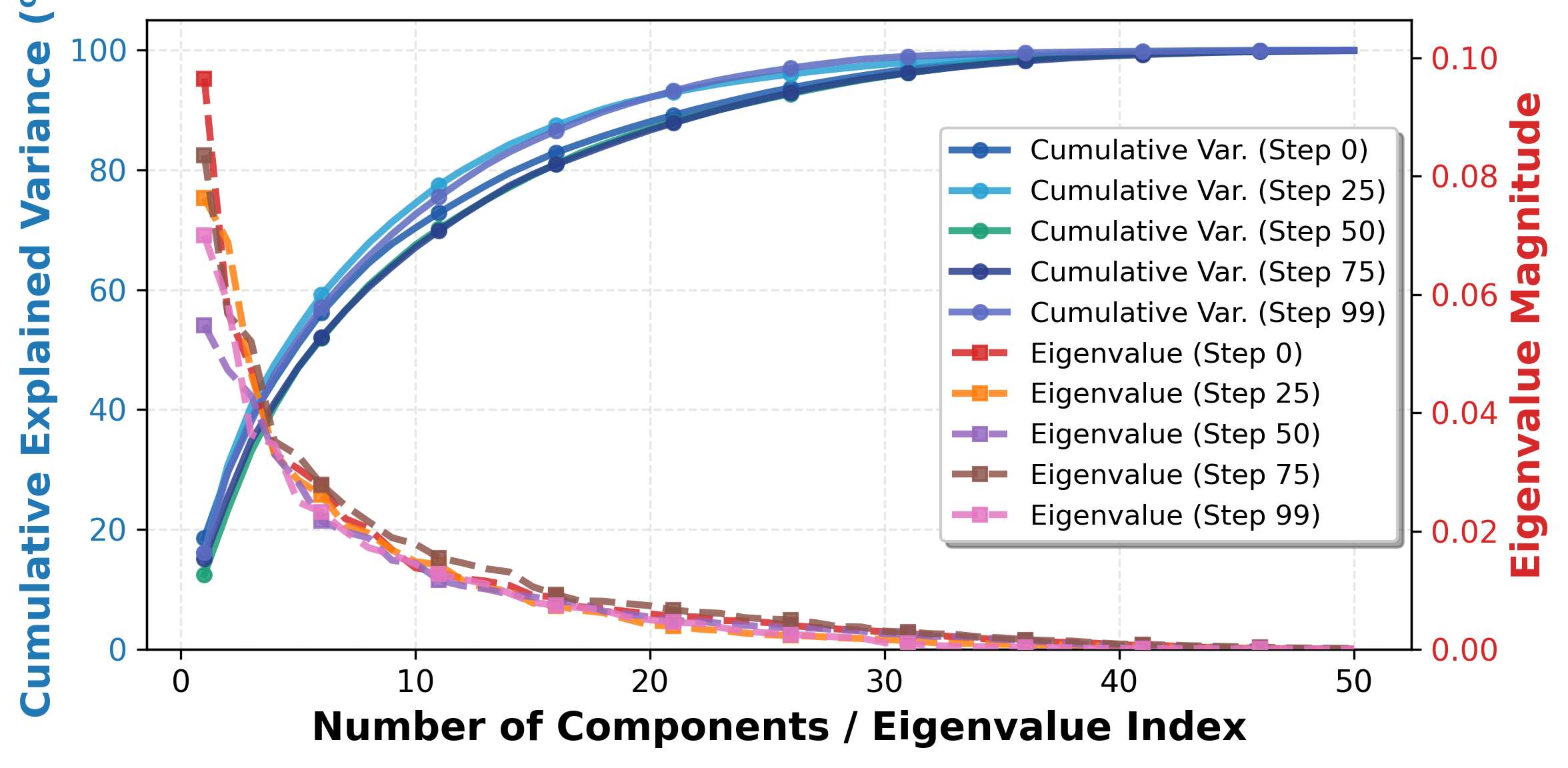}
    \caption*{\small (a) Low-rank Hessian spectrum}
  \end{minipage}

  \vspace{5pt} 

  \begin{minipage}[t]{\linewidth}
    \centering
    \includegraphics[width=\linewidth]{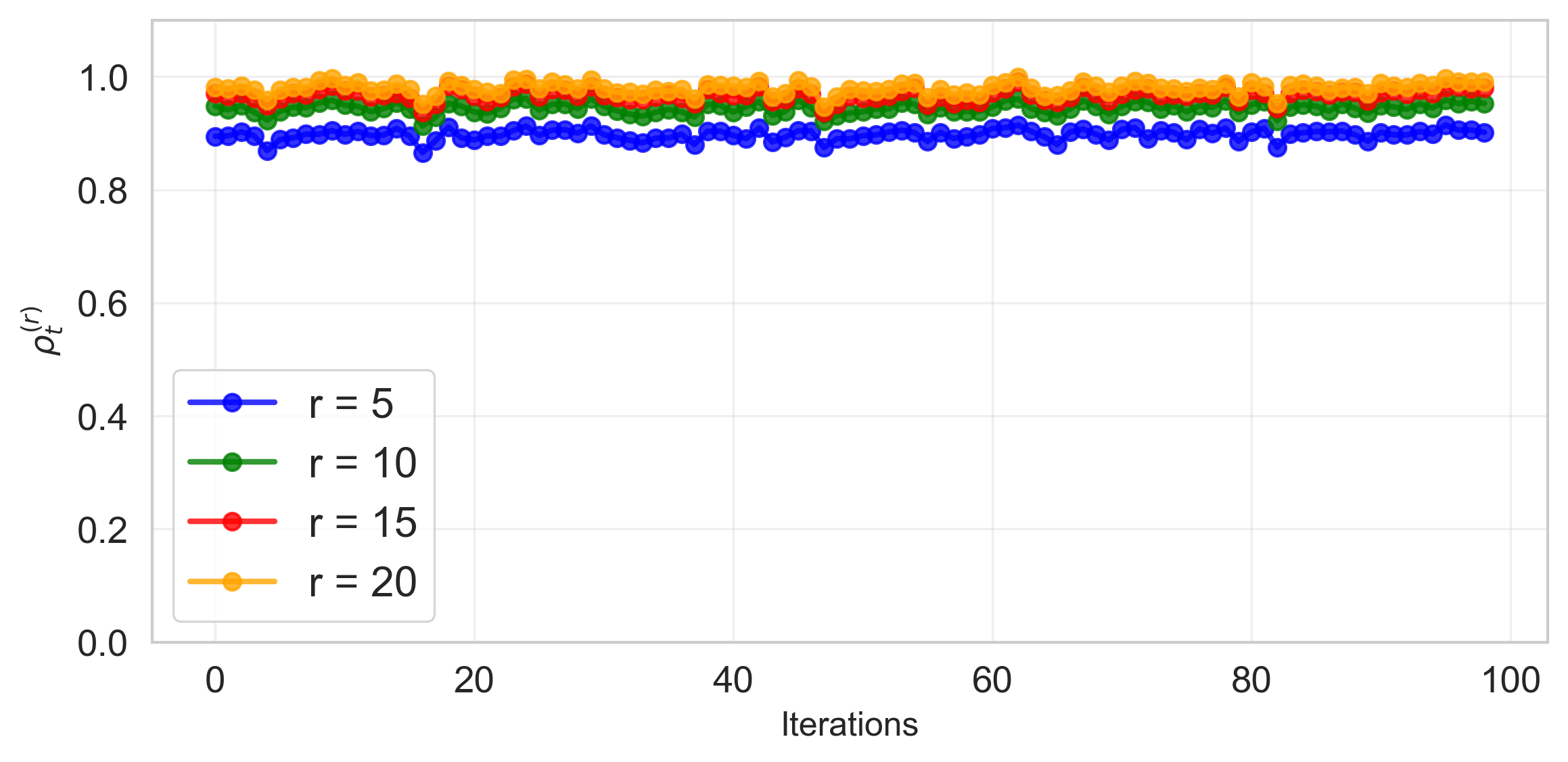}
    \caption*{\small (b) Slow-varying subspace $\rho_t^{(r)}$}
  \end{minipage}

  \caption{Hessian analysis for the composite loss $\mathcal{L}_{\text{com}}$.}
  \label{fig:hessian_composite_loss}
\end{figure}

\subsection{Adapter Layer Selection Across Transformer Backbones}
\label{app:layer_selection_multiarch}

In the main paper, we select an early layer (layer-3) to insert the adapter for adaptation efficiency.
To verify the feasibility and generality of this choice beyond ViT-B/16, we conduct additional experiments on DeiT\cite{touvron2021deit} and Swin-Tiny\cite{liu2021Swin}.
As shown in Fig.~\ref{fig:layer_ablation_multiarch}, we observe a consistent trend across architectures: inserting the adapter into relatively
early layers (typically the 3rd or 4th block) yields the best performance.
Therefore, selecting layer-3 (or layer-4 when preferred) provides a practical and universal default for different Transformer backbones.

\begin{figure}[htbp]
  \centering
  \includegraphics[width=\linewidth]{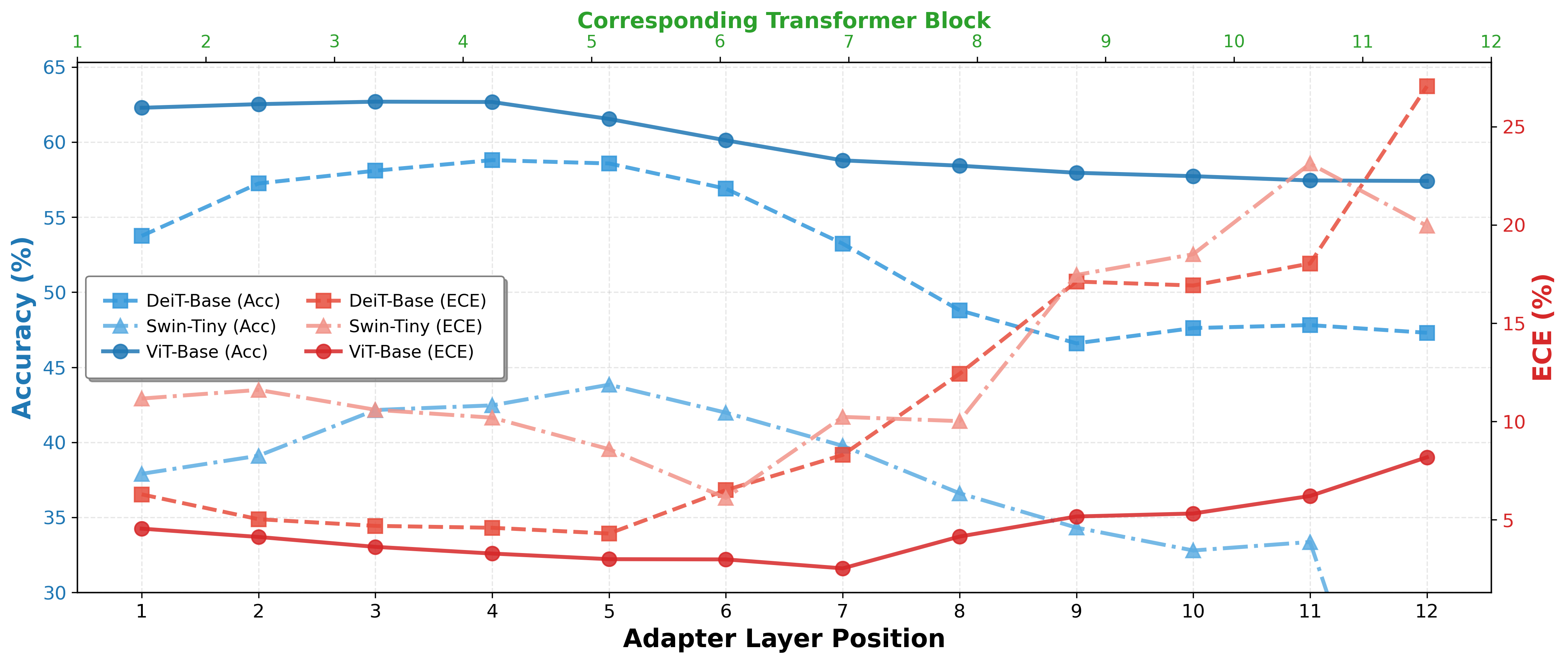}
   \caption{Adapter layer position ablation across Transformer backbones (ViT, DeiT, Swin-Tiny). Early layers (e.g., 3rd/4th) tend to perform best.}
   \label{fig:layer_ablation_multiarch}
\end{figure}

\subsection{Controlled Experiments on Optimization Efficiency Under TTA Constraints}
\label{app:tta_constraints_efficiency}

Test-time adaptation is a constrained unsupervised setting where each test sample is seen only once,
and the adaptation must be completed within a limited number of update steps. Therefore, optimization efficiency
under a limited update budget is crucial.
We conduct controlled experiments comparing (i) standard backpropagation-based adaptation (BP),
(ii) vanilla zeroth-order optimization (ZO; RGE), and (iii) CAZO, under:
\emph{(a) unsupervised TTA} and \emph{(b) supervised adaptation (fine-tuning)} with the same number of parameter updates.

As expected, BP adapts fastest. Vanilla ZO suffers from slow convergence due to the well-known variance scaling with dimension $\mathcal{O}(d)$,
which becomes more severe under the short-horizon TTA constraint. By leveraging curvature information through anisotropic perturbation sampling,
CAZO reduces the variance of ZO updates and achieves much faster learning, leading to significant improvement over vanilla ZO in the limited-step TTA regime.
Figure~\ref{fig:tta_constraints} summarizes the learning dynamics under these controlled settings.

\begin{figure}[htbp]
  \centering
  \includegraphics[width=\linewidth]{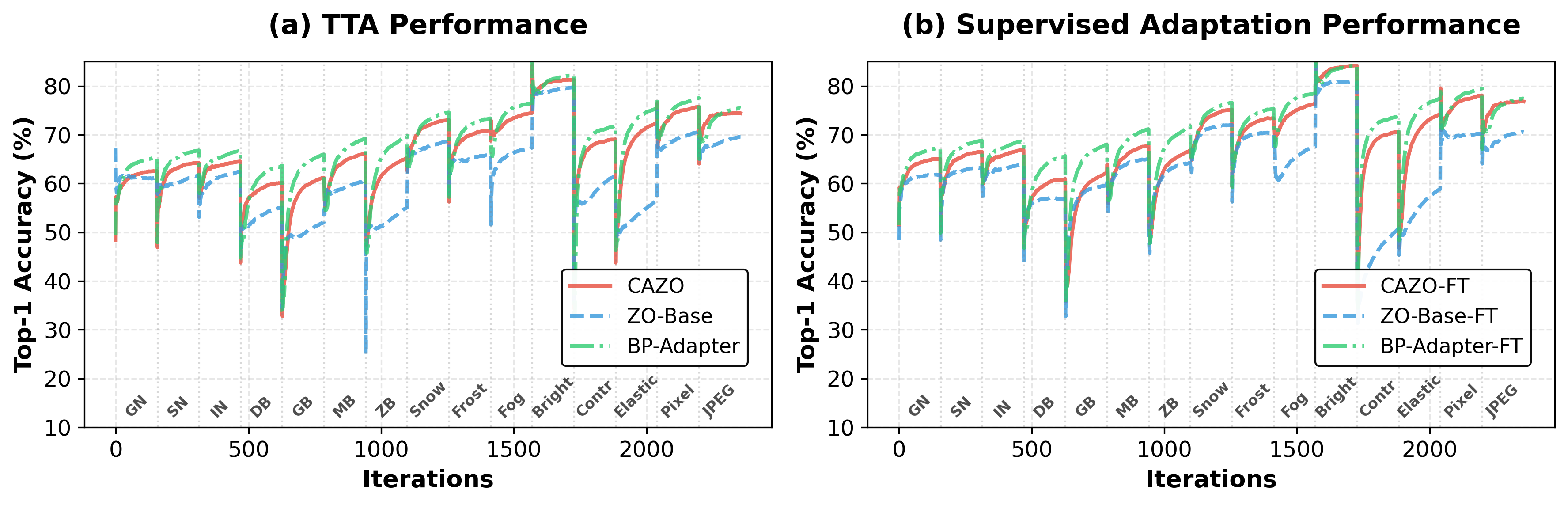}
   \caption{Controlled experiments quantifying the impact of TTA constraints on optimization efficiency:
(a) unsupervised TTA; (b) supervised adaptation with the same update budget.}
   \label{fig:tta_constraints}
\end{figure}

\subsection{Extended Ablation Study}
\paragraph{Downsampling Ratio} We evaluate the effect of the downsampling ratio of the adapter. Table~\ref{tab:downsampling_cazo} presents the results of different downsampling ratios of the adapter when it is applied to the first transformer layer of ViT. 

\begin{table}[htbp]
\centering
\footnotesize
\caption{Adapter downsampling with different ratios in CAZO. The values in parentheses represent the downsampled dimensions.}
\resizebox{\columnwidth}{!}
{
\begin{tabular}{c|cccccc}
\hline
\textbf{Downsampling Ratio} & 384(2) & 256(3) & 192(4) & 128(6) & 96(8) & 48(16) \\
\hline
\textbf{Accuracy} & 62.7 & 62.6 & 62.8 & 62.7 & 62.5 & 62.3 \\
\textbf{ECE} & 3.28 & 3.20 & 3.14 & 3.29 & 3.27 & 3.21 \\
\hline
\end{tabular}
}
\label{tab:downsampling_cazo}
\end{table}

\paragraph{EMA Update Interval}
We study the sensitivity to the EMA update interval $\tau$ for the curvature proxy.
Specifically, $\tau$ denotes the interval (in steps) at which we update the EMA-based diagonal curvature estimate;
larger $\tau$ reduces the update frequency but may lag behind curvature drift.
Table~\ref{tab:sensitivity_tau} shows that $\tau \in \{1,2,3\}$ yields nearly identical performance,
while overly infrequent updates (e.g., $\tau=10$) lead to a noticeable degradation.
We use $\tau=1$ by default.
\begin{table}[htbp]
\centering
\caption{Ablation on EMA update interval.}
\label{tab:sensitivity_tau}
\resizebox{0.7\linewidth}{!}{%
\begin{tabular}{@{}lccccc@{}}
\toprule
\multicolumn{6}{c}{Update interval $\tau$} \\
\midrule
$\tau$ & 1 & 2 & 3 & 5 & 10 \\
Acc. & \textbf{69.01} & 69.00 & 68.98 & 68.80 & 68.34 \\
\bottomrule
\end{tabular}
}
\end{table}

\paragraph{Perturbation Scale}
We also evaluate the perturbation magnitude $\epsilon$ used in symmetric finite-difference gradient estimation.
As shown in Table~\ref{tab:sensitivity_eps}, a moderate $\epsilon$ (around $0.1$) achieves the best accuracy.
Too small $\epsilon$ can be dominated by numerical noise and yield weak signal, while too large $\epsilon$ introduces bias by probing
a non-local region of the loss landscape. We use $\epsilon=0.1$ in all experiments.

\begin{table}[htbp]
\centering
\caption{Ablation on perturbation scale.}
\label{tab:sensitivity_eps}
\resizebox{0.7\linewidth}{!}{%
\begin{tabular}{@{}lccccccc@{}}
\toprule
\multicolumn{8}{c}{Perturbation scale $\epsilon$} \\
\midrule
$\epsilon$ & 0.01 & 0.05 & 0.08 & 0.10 & 0.20 & 0.50 & 1.00 \\
Acc. & 61.01 & 67.75 & 68.82 & \textbf{69.01} & 61.78 & 49.25 & 44.08 \\
\bottomrule
\end{tabular}
}
\end{table}

\paragraph{Exploration of Visual Prompt Tuning (VPT)} We also consider adapting VPT and evaluate the differences among different parameter efficient fine-tuning methods as in Table~\ref{tab:VPT_results}. 
\begin{table*}[htbp]
\centering
\caption{Performance comparison of VPT versus Adapter methods on ImageNet-C (severity level 5) with ViT: Accuracy (\%, $\uparrow$). Batch size is fixed at 64.}
\resizebox{\textwidth}{!}{%
\begin{tabular}{l|ccc|cccc|cccc|cccc|cc}
\hline
& \multicolumn{3}{c|}{Noise} & \multicolumn{4}{c|}{Blur Defoc.} & \multicolumn{4}{c|}{Weather} & \multicolumn{4}{c|}{Digital} & {Average} \\
Method & Gauss. & Shot & Impul. & Defoc. & Glass & Motion & Zoom & Snow & Frost & Fog & Brit. & Contr. & Elas. & Pix. & JPEG & Acc.\\
\hline
NoAdapt &  56.8 & 56.8 & 57.5 & 46.9 & 35.6 & 53.1 & 44.8 & 62.2 & 62.5 & 65.7 & 77.7 & 32.6 & 46.0 & 67.0 & 67.6 & 55.5 \\
ZO (Adapter) & 61.7 & 62.5 & 63.2 & 54.2 & 51.9 & 59.9 & 52.7 & 67.9 & 67.7 & 68.8 & 79.8 & 65.5 & 57.4 & 70.7 & 71.1 & 63.6 \\
FOA (VPT) & 61.6 & 62.5 & 63.2 & 58.5 & 54.0 & 61.2 & 57.0 & 69.6 & 68.6 & 74.1 & 80.9 & 67.3 & 63.3 & 73.6 & 72.7 & 65.9 \\
\rowcolor{gray!10}
CAZO (VPT)&61.8&62.7&63.1&58.0 &51.6&62.6&57.4&70.5&66.2&73.2&81.2&67.8&68.4&74.7&72.3&66.1\\
\rowcolor{gray!10}
CAZO (Adapter) & 62.7 & 64.3 & 64.1 & 60.0 & 61.5 & 66.0 & 64.8 & 72.7 & 71.0 & 74.3 & 81.3 & 69.5 & 72.7 & 75.7 & 74.5 & 69.0\\
\hline
\end{tabular}
}%
\label{tab:VPT_results}
\end{table*}

Our comprehensive comparison between visual prompt tuning (VPT) and adapter-based efficient tuning methods reveals no discernible performance advantage for the VPT paradigm. Under equivalent parameter budgets (VPT implemented with 3 prompts), CAZO-VPT (66.1\%) exhibits lower average accuracy than CAZO-Adapter (69.0\%) on ImageNet-C severe corruption benchmarks. This performance gap persists across all corruption categories including noise, blur, weather, and digital distortions. The consistent accuracy deficit relative to adapting adapter informs our methodological selection. Consequently, we adopt adapter as our primary tuning framework, prioritizing its superior corruption robustness as evidenced by the comprehensive benchmark results.

\paragraph{Different $\nu$ Influence in CAZO} Table~\ref{tab:hessian_nu} shows the performance of CAZO under different $\nu$ values (with $\epsilon=0.1$). The results are reported on ImageNet-C (severity level 5) with ViT-B/16, and both accuracy (Acc.) and expected calibration error (ECE) are averaged over all corruptions. 

\begin{table}[htbp]
\centering
\caption{Performance of CAZO for different $\nu$ on ImageNet-C (severity level 5) with ViT-B/16.}
\begin{tabular}{c|c|c}
\hline
$\nu$ & Acc. (\%, $\uparrow$) & ECE (\%, $\downarrow$) \\
\hline
0.01 & 68.3 & 4.6 \\
0.05 & 68.5 & 4.5 \\
0.1 & 68.5 & 4.5 \\
0.2 & 68.5 & 4.4 \\
0.3 & 68.6 & 4.4 \\
0.4 & 68.6 & 4.4 \\
0.5 & 68.7 & 4.4 \\
0.6 & 68.9 & 4.3 \\
0.7 & 68.9 & 4.3 \\
0.8 & 69.0 & 4.3 \\
0.9 & 68.8 & 4.2 \\
0.95 & 64.4 & 5.0 \\
1 & 0.1 & * \\
\hline
\end{tabular}
\label{tab:hessian_nu}
\end{table}
The results in Table~\ref{tab:hessian_nu} show that choosing a value of $\nu$ in the range from 0.6 to 0.9 can yield satisfactory performance. We use $\nu = 0.8$ in all experiments.


    
    